\documentclass[sigconf, nonacm]{acmart}

\usepackage{hyperref} 
\usepackage{graphicx}
\usepackage{booktabs}
\usepackage{amsmath}
\usepackage{amsthm}
\usepackage{bbm}
\usepackage{amsfonts}
\usepackage{graphicx}
\usepackage{subfigure}
\usepackage{url}
\usepackage{diagbox}
\usepackage{balance}
\usepackage[noend]{algorithmic}
\usepackage{algorithm}
\usepackage{multirow}
\usepackage{xspace}

\usepackage{empheq}

\usepackage{pifont}
\newtheorem{theor}{Theorem}

\newtheorem{lem}{Lemma}

\newtheorem{example}{Example}

\newcommand{\squishlist}{
 \begin{list}{$\bullet$}
  {  \setlength{\itemsep}{0pt}
     \setlength{\parsep}{3pt}
     \setlength{\topsep}{3pt}
     \setlength{\partopsep}{0pt}
     \setlength{\leftmargin}{2em}
     \setlength{\labelwidth}{1.5em}
     \setlength{\labelsep}{0.5em}
} }
\newcommand{\squishlisttight}{
 \begin{list}{$\bullet$}
  { \setlength{\itemsep}{0pt}
    \setlength{\parsep}{0pt}
    \setlength{\topsep}{0pt}
    \setlength{\partopsep}{0pt}
    \setlength{\leftmargin}{2em}
    \setlength{\labelwidth}{1.5em}
    \setlength{\labelsep}{0.5em}
} }

\newcommand{\squishdesc}{
 \begin{list}{}
  {  \setlength{\itemsep}{0pt}
     \setlength{\parsep}{3pt}
     \setlength{\topsep}{3pt}
     \setlength{\partopsep}{0pt}
     \setlength{\leftmargin}{1em}
     \setlength{\labelwidth}{1.5em}
     \setlength{\labelsep}{0.5em}
} }

\newcommand{\squishend}{
  \end{list}
}

\newcommand{\eat}[1]{}

\newcommand{\ie}{i.e.,\xspace}
\newcommand{\eg}{e.g.,\xspace}
\newcommand{\wrt}{w.r.t.\xspace}

\newcommand{\NP}{\ensuremath{\mathbf{NP}}\xspace}

\newcommand{\kw}[1]{{\ensuremath {\mathsf{#1}}}\xspace}

\newcommand{\stitle}[1]{\vspace{1.5ex}\noindent{\bf #1}}

\newcommand{\eetitle}[1]{\vspace{0.8ex}\noindent{\em\underline{#1}}}

\newcommand{\stab}{\rule{0pt}{8pt}\\[-1.6ex]}

\newcommand{\sstab}{\rule{0pt}{8pt}\\[-2.4ex]}
\newcounter{ccc}

\DeclareMathOperator*{\argmax}{arg\,max}

\newcommand\redout{\bgroup\markoverwith
{\textcolor{red}{\rule[.5ex]{2pt}{2pt}}}\ULon}

\newcommand{\G}{{\mathcal G}}
\newcommand{\V}{{\mathcal V}}

\newcommand{\M}{{\mathcal M}}
\newcommand{\C}{{\mathcal C}}
\newcommand{\gnn}{\kw{GNN}}
\newcommand{\gnns}{\kw{GNNs}}

\newcommand{\gcn}{\kw{GCN}}
\newcommand{\gcns}{\kw{GCNs}}

\newcommand{\PTIME}{\kw{PTIME}}

\renewenvironment{proof}{
        \vspace{1ex}
        {\noindent\bf Proof:}}{\vspace{1ex}}
\newenvironment{proofS}{
        \vspace{1ex}
        {\noindent\bf Proof sketch:\ }}{\vspace{1ex}}

\newcommand{\slice}{\kw{SliceGX}}

\renewcommand{\S}{\mathcal{S}}
\newcommand{\verify}{\kw{Verify}}

\newcommand{\slicess}{\kw{Slice(SS)}}
\newcommand{\slicems}{\kw{Slice(MS)}}
\newcommand{\slicemm}{\kw{Slice(MM)}}

\newcommand{\verifyss}{\kw{Verify(SS)}}
\newcommand{\verifyms}{\kw{IncVerify(MS)}}

\newcommand{\gnnexp}{\kw{GNNExplainer}}
\newcommand{\subx}{\kw{SubgraphX}}
\newcommand{\pgexp}{\kw{PGExplainer}}
\newcommand{\gmask}{\kw{GraphMask}}
\newcommand{\same}{\kw{SAME}}
\newcommand{\flowx}{\kw{FlowX}}

\newcommand{\ba}{\kw{BA}}

\newcommand{\am}{\kw{AM}}

\newcommand{\ttrevise}[1]{{\color{black}{#1}}}

 \newcommand{\bi}{\begin{itemize}}
\newcommand{\ei}{\end{itemize}}
\newenvironment{tbi}{\begin{itemize}
        \setlength{\topsep}{1.5ex}\setlength{\itemsep}{0ex}}
        {\end{itemize}\vspace{-0.5ex}}
\usepackage[utf8]{inputenc}
\usepackage{soul}
\AtBeginDocument{%
  }

\usepackage{amsmath,amsfonts} 
\usepackage{graphicx}
\usepackage{textcomp}
\usepackage{xcolor}
\usepackage{tabularx}
\usepackage{booktabs}
\usepackage{array}
\newcolumntype{C}{>{\centering\arraybackslash}X} 
\usepackage{graphicx}
\usepackage{subcaption}

\copyrightyear{2026}
\acmYear{2026}
\setcopyright{cc}
\setcctype{by}
\acmConference[WWW '26] {Proceedings of the ACM Web Conference 2026}{April 13--17, 2026}{Dubai, United Arab Emirates.}
\acmBooktitle{Proceedings of the ACM Web Conference 2026 (WWW '26), April 13--17, 2026, Dubai, United Arab Emirates}
\acmISBN{979-8-4007-2307-0/2026/04}
\acmDOI{10.1145/3774904.3792689} 
\begin{document}

\title{SliceGX: Layer-wise GNN Explanation with Model-slicing}

\author{Cibo Yu}
\affiliation{%
  \institution{Zhejiang University}
  \city{Hangzhou}
  \country{China}
}
\email{cbyu@zju.edu.cn}

\author{Tingting Zhu}
\affiliation{%
  \institution{Zhejiang University}
  \city{Hangzhou}
  \country{China}
}
\email{tingtingzhu@zju.edu.cn}

\author{Tingyang Chen}
\affiliation{%
  \institution{Zhejiang University}
  \city{Hangzhou}
  \country{China}
}
\email{chenty@zju.edu.cn}

\author{Yinghui Wu}
\affiliation{%
  \institution{Case Western Reserve University}
  \city{Cleveland}
  \country{USA}
}
\email{yxw1650@case.edu}

\author{Arijit Khan}
\affiliation{%
  \institution{Bowling Green State University}
  \city{Bowling Green}
  \country{USA}
}
\email{arijitk@bgsu.edu}

\author{Xiangyu Ke}
\authornote{Corresponding author.}
\affiliation{%
  \institution{Zhejiang University}
  \city{Hangzhou}
  \country{China}
}
\email{xiangyu.ke@zju.edu.cn}
\renewcommand{\shortauthors}{Cibo Yu et al.}

\begin{abstract}
Ensuring the trustworthiness of graph neural networks (\gnns), which are often treated as black-box models, requires effective explanation techniques. 
Existing \gnn explanations typically apply input perturbations to identify subgraphs that are responsible for the occurrence of the final output of \gnns. 
However, such approaches lack finer-grained, layer-wise analysis of how intermediate representations contribute to the final result, capabilities that are crucial for model diagnosis and architecture optimization. 
This paper introduces \slice, a novel \gnn explanation approach that generates explanations at specific \gnn layers in a progressive manner. 
Given a \gnn model $\M$, a set of selected intermediate layers, and a target layer, \slice  slices $\M$ into layer blocks (``model slice'') and discovers high-quality explanatory subgraphs within each block that elucidate how the model output arises at the target layer.
Although finding such layer-wise explanations is computationally challenging, we develop efficient algorithms and optimization techniques that incrementally construct and maintain these subgraphs with provable approximation guarantees. 
Extensive experiments on synthetic and real-world benchmarks demonstrate the effectiveness and efficiency of \slice, and illustrate its practical utility in supporting model debugging.
\end{abstract}



\maketitle

\section{Introduction}
\label{sec:intro}

Graph Neural Networks (\gnns) have demonstrated promising performances in various applications, such as recommender systems~\cite{chen2023heterogeneous}, knowledge graph reasoning~\cite{zhang2023adaprop}, and social network analysis~\cite{sankar2021graph}.
A \gnn is a model $\M$ that converts an input graph $G$ and its initial node feature matrix $Z^0$ into a new numerical matrix $Z$, which can be further post-processed to produce task-specific answers, such as node classification and graph classification.

\vspace{.5ex}
However, predictive accuracy alone does not guarantee the trustworthiness and reliability of graph learning; transparency of the inference process is also required~\cite{arenas2024data}. 
Understanding how a \gnn arrives at a decision is therefore essential for model provenance and tracking. For example, a provenance-oriented analysis~\cite{arenas2024data} may ask ``\textit{What/Which}'' questions to identify which input features are relevant, ``\textit{How}'' questions to trace how features transform across layers, or ``\textit{Where/When}'' questions to determine which layers contribute most during inference. 
Such finer-grained, ``layer-wise'' interpretation, as also practiced in other deep models, e.g., in CNNs~\cite{iqbal2022visual} and DNNs~\cite{ullah2021explaining}, further benefits model optimization and debugging~\cite{adebayo2020debugging}.

\vspace{.5ex}
Despite the evidential need, the study of layer-wise interpretation for \gnns, and how it benefits \gnn applications, still remains in its infancy. Consider the following example.

\begin{figure}[tb!]
\vspace{1ex}
\centering
\centerline{\includegraphics[scale=0.18]{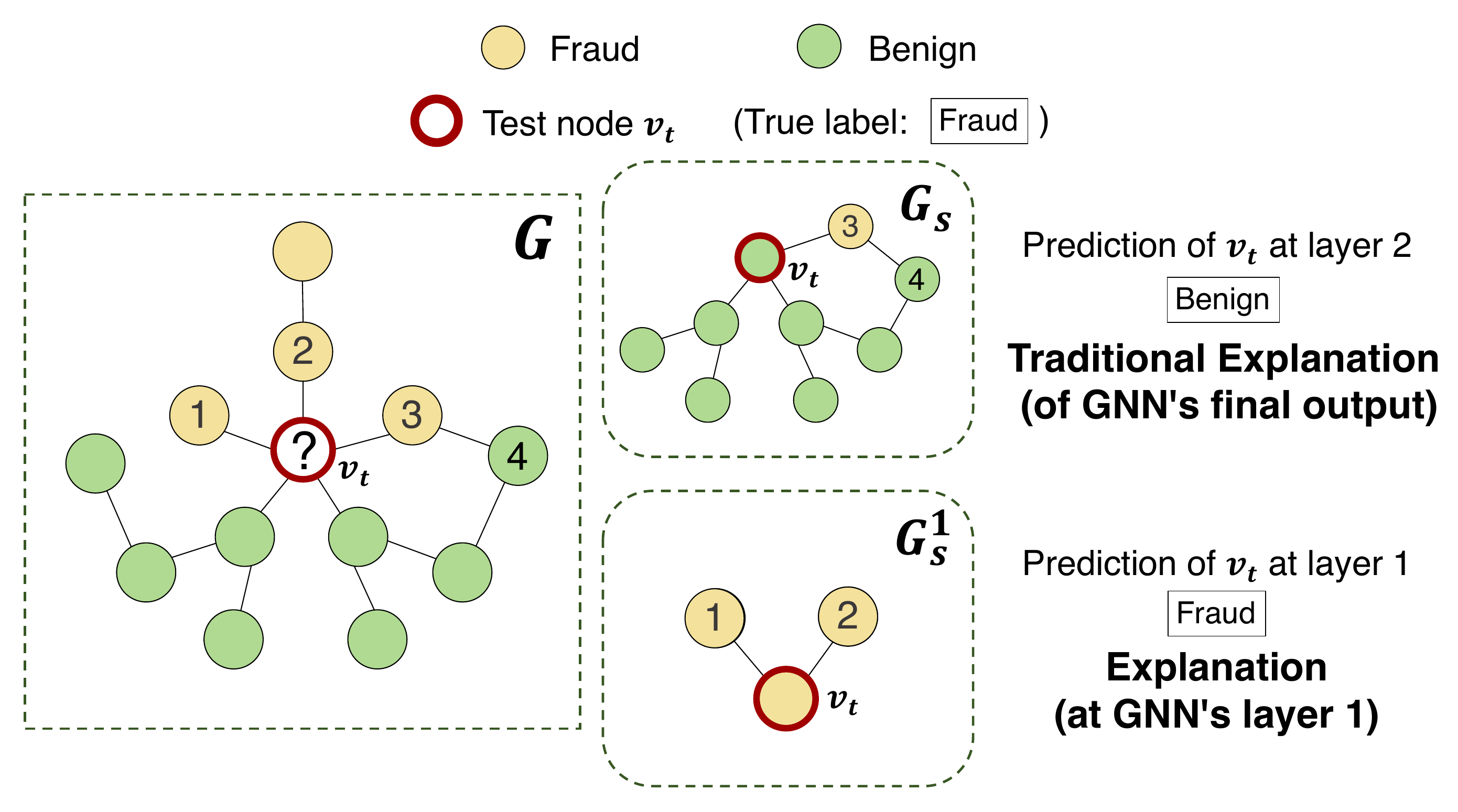}}
\vspace{-3ex}
\caption{Generating layer-wise explanations for GNN diagnosing for spam review detection~\cite{dou2020enhancing}.}
\label{intro}
\vspace{-3ex}
\end{figure}

\begin{example}
\label{exa-motivation}
\vspace{-1mm}
A fraction of a spam review network $G$~\cite{dou2020enhancing} is illustrated in Figure ~\ref{intro}, where nodes are reviews and edges indicate associations such as shared reviewers, common products, or temporal links. 
A \gnn classifies reviews as either ``fraud'' (yellow) or ``benign'' (green), with fraud reviews exhibiting deceptive or spam-like behavior.

Observing that a fraud review $v_t$ is incorrectly labeled as ``benign'', the analyst may want to diagnose the inference process of the \gnn $\M$. 
Several ``diagnose'' queries, analogous to those in CNNs and DNNs~\cite{iqbal2022visual,ullah2021explaining}, can be posed:
\tbi 
\item ``\textbf{Q1}: {\em At which layer (When) did $\M$ first err}?'' Identifying the first erroneous layer enables ``stop-on-first-error'' inspection and early intervention.; 
\item ``\textbf{Q2}: {\em Which fraction of $G$ (subgraphs) are responsible for this first mistake}?'' 
This reveals the nodes, edges, and features that directly influence the decision;
\item ``\textbf{Q3}: {\em How did the influence propagate via subsequent layers}?'' 
This traces how the identified influences evolve and spread during inference.
\ei

A closer inspection of $\M$'s inference process reveals the following.

\sstab 
(1) For \textbf{Q1}, by treating each layer as an ``output layer'', we can ``segment'' the inference process and analyze the node embedding of $v_t$ at a specific layer to assess if the accumulated neighborhood impact leads to an incorrect label. 
The established strategy that aggregates localized neighbors such as ``soft labels''~\cite{yang2025soft} to track node influence also justifies this intuition. 
For example, $v_t$ is correctly classified as ``fraud'' at layer 1 but is misclassified as ``benign'' at layer 2, alerting a first error at layer $2$ for further investigation. 

\sstab 
(2) For \textbf{Q2}, existing \gnn explainers~\cite{ying2019gnnexplainer,yuan2021explainability,huang2022graphlime,schlichtkrull2020interpreting,gui2023flowx} such as GNNExplainer~\cite{ying2019gnnexplainer} (see Related Work) provide factual or counterfactual subgraph explanations that are ``faithful'' to the label. 
However, they typically focus on the final output and do not adapt naturally to intermediate layers as derived for \textbf{Q1}. 
 The influential subgraph at layer 1, $G^1_s$, can differ substantially from that at layer 2, $G_s$ (in Figure~\ref{intro}), reflecting distinct local influences and feature structures that explain the layer-specific predictions.

\sstab 
(3) To answer \textbf{Q3}, one needs to also necessarily annotate the explanation structures by distinguishing between influencing nodes and edges that contribute most to the output at targeted layers and those that are only “connecting” and merely propagating such influence during the inference. 
For example, the graph $G_s$ may indicate that the first mistake is due to accumulated influence from a majority of ``benign'' 2-hop neighbors of $v_t$ at layer $2$, while an adjacent edge of $v_t$ merely propagates this influence (e.g., node $3$ propagates information from node $4$), hence suggesting further investigation for the ``How'' question \textbf{Q3}. Prior \gnn explainers generally do not separate  these roles.
\end{example}
\vspace{-0.5ex}

The above example calls for a layer-wise explanation framework for \gnns that can efficiently produce concise, consistent explanations for intermediate-layer outputs during inference. 
We aim to address the following questions:  
(1) {\em How to define explanation structures for layer-wise \gnn explanations?} 
(2) {\em How to measure the explainability of layer-wise explanations?} 
and (3) {\em How to efficiently generate layer-wise \gnn explanations?}
Existing \gnn explainability methods do not fully answer these questions. 

\vspace{1ex}
Our main contributions are as follows.
\vspace{-1ex}
\begin{itemize}
    \item We extend \gnn explanation to its layer-wise counterpart. To this end, we introduce the "model slice" concept and a novel explanatory subgraph structure, which together enable precise analysis of a \gnn's layer-wise influence.
    
    \item We present \slice, a layer-wise \gnn explanation framework to generate explanatory subgraphs at different layers. To guide the generation process, we propose a bi-criteria quality measure that integrates message-passing influence and embedding diversity. 

    \item We present training-free post-hoc explanation generation algorithms to discover connected subgraph explanations for specific model layers. We propose an efficient $\frac{1}{2}$-approximation generation algorithm for single-node single-source layers and generalize it to multi-node multi-source settings. 

    \item We validate \slice on six benchmarks, demonstrating its efficiency and effectiveness. Crucially, it uncovers \gnn insights that facilitate downstream tasks like debugging and optimization.
\end{itemize} 
\section{Related Work}
\vspace{-1ex}

\stitle{Explainability of GNNs}. 
Several approaches have been proposed to explain \gnns' predictions~\cite{ying2019gnnexplainer,yuan2021explainability,huang2022graphlime,schlichtkrull2020interpreting,gui2023flowx, huang2025sehg, zhang2024adversarial}, primarily by identifying important nodes, edges, node features, subgraphs, and message flows. 
A common strategy is to monitor the changes in prediction under input perturbations. 
For example, GNNExplainer~\cite{ying2019gnnexplainer} learns soft masks over edges and node features to maximize mutual information between original and masked graph predictions. \subx~\cite{yuan2021explainability} efficiently samples explanatory subgraphs using pruning-based Monte Carlo Tree Search (MCTS) together with Shapley-value estimates. 
In contrast, SAME~\cite{ye2023same} extends this with a two-phase framework that first identifies multiple important substructures and then aggregates them into a final explanation. 
Other lines of work use surrogate models to approximate \gnn predictions and derive explanations from the surrogate~\cite{huang2022graphlime, duval2021graphsvx}. 

Despite their strengths, these methods are not directly suitable for layer-wise explanations. 
First, they typically focus on final outputs, failing to capture the progressive decision formation across layers. Second, they do not distinguish between influence sources and structural propagators. Third, they cannot explicitly explain intermediate-layer outputs..

\vspace{-1mm}
\stitle{Layer-wise GNN Analysis}. 
A smaller body of work studies layer-wise attribution for \gnns. 
Decomposition-based methods, \ie {\sf GNN-LRP}~\cite{schnake2021higher} and {\sf Excitation-BP}~\cite{pope2019explainability}, assign importance scores per node at each layer by backpropagating relevance. 
{\sf GraphMask}~\cite{schlichtkrull2020interpreting}, a post-hoc interpretation method, produces differentiable edge masks for each layer but focuses on explaining the final prediction rather than intermediate targets. 
\flowx~\cite{gui2023flowx} models explanations as message flows, which are sequences of layer edges, allowing it to calculate the importance of specific message-passing steps at each individual layer.
While valuable, these methods share a limitation: they interpret intermediate layers relative to the final prediction rather than providing self-contained explanations for intermediate outputs.

\vspace{-1mm}
\stitle{Model Slicing}. Model slicing, with 
a counterpart in program slicing~\cite{xu2005brief}, extracts sub-models for finer-grained analysis ~\cite{blouin2011modeling} and has been applied to efficient inference~\cite{cai2019model, zeng2024powering}, adversarial  detection~\cite{zhang2020dynamic}, mechanistic interpretability \cite{RaukerHCH23}, and visual analytics~\cite{hsu2019couper}. 
This paper makes a first step to adapt model-slicing to the \gnn setting and develop efficient algorithms to generate layer-specific explanations that are both concise and annotated with functional roles. 

\vspace{-1ex}
\section{Graph Neural Networks and Explanation}
\label{sec:preliminaries}
\vspace{-1ex}

\stitle{Graphs}.  
A graph $G = (V,E)$ has a set of nodes $V$ and a set of edges $E$. Each node $v$ carries a tuple, representing its attributes (or features) and their corresponding values. 

We define the following terms: 
(1) The $l$-hop neighbors of a node $v \in V$, denoted $N^l(v)$, are nodes {\em within} $l$ hops from $v$ in $G$. 
(2) The $l$-hop neighbor subgraph, $G^l(v)$, is a node induced subgraph of $G$, 
which is induced by $N^l(v)$. 
(3) For a set of nodes $V_s \subseteq V$, the $l$-hop neighbor subgraph, $G^l(V_s)$, is induced by the union of $l$-hop neighbors of nodes in $V_s$, i.e., $\bigcup_{v \in V_s} N^l(v)$. $G^l(v)$ is always connected, while $G^l(V_s)$ may not be. 
(4) The size of a graph is the total number of nodes and edges, given by $|G| = |V| + |E|$.

\vspace{-0.8ex}
\stitle{Graph Neural Networks}.
\gnns are a class of neural networks that learn to convert an input graph $G$ and its initial node feature matrix $Z^0$ to proper representations. We provide an abstract characterization of a \gnn $\M$ with $L$ layers. 

In general, a \gnn can be considered as a composition 
of two functions: an encoder $f_1$ and a predictor $f_2$.

\sstab 
{\bf (1)} The encoder $f_1$ is a composite function 
$f^L_1\circ f^{L-1}_1\circ \ldots \circ f^1_1(G)$. 
At the $l$-th layer ($l\in [1,L]$), the function $f^l_1$ 
uniformly updates, for each node $v\in V$, 
an input representation from the last layer as: 
\begin{equation}
\vspace{-1ex}
\label{eq-prop}
    z^l_v := \kw{TRANS}\left(\kw{AGG}^l\left({z^{l-1}_j\mid j\in N^l(v)\cup \{v\}}\right)\right)
\end{equation}

Here, $z^l_v$ is the feature representation (``embedding'') of node $v$ at the $l$-th layer ($z^0_v$ is the input feature of $v$); $N_v$ is the set of neighbors of node $v$ in $G$, and $AGG^l$ (resp. $TRANS^l$) represents an aggregation (resp. transformation) function at the $l$-th layer, respectively.

\sstab 
{\bf (2)} Subsequently, a predictor $f_2$ takes the output embeddings $f_1(G)$ = $Z$ and performs a task-specific post-processing $f_2(Z)$, to obtain a final output (a matrix), denoted as $\M(G)$. Here, $f_2$ typically includes a multi-layer perceptron (MLP) that takes the aggregated representations $Z_v$ from the encoder $f_1$, and produces the final output for the specific task, such as node classification or graph classification. Specifically, given a node $v\in V$, the output of $\M$ at $v$, denoted as $\M(G,v)$, refers to the entry of $f_2(Z)$ (\ie $\M(G)$) at node $v$.

\vspace{.5ex}
The above step is an abstraction of the inference process of $\M$. Variants of \gnns instantiate different aggregation and transformation functions to optimize task-specific output. 
A canonical example is the Graph Convolutional Network (\gcn)~\cite{kipf2016semi}, which can be viewed as a first-order spectral approximation that aggregates (normalized) neighbor embeddings and applies a nonlinear transformation, \eg ReLU, as the transformation operator.


\stitle{Explanation Structure of \gnns}.
We start with a notion of explanation. Given a \gnn $\M$ with $L$ layers, a graph $G=(V,E)$, and an output $\M(G,v_t)$ to be explained, an {\em explanation} of $\M(G,v_t)$ is a node-induced, connected subgraph 
$G_s$ = $(V_s\cup V_c, E_s)$ of $G$, where 
\bi 
\item (1) $V_s \subseteq V$ is a set of {\em explanatory nodes} that clarify the output $\M(G,v_t)$ ($v_t \in V_s$)
\item (2) $V_c \subseteq V$ is a {\em connector set} induced by explanatory nodes $V_s$, which is defined as $\bigcup_{v_s \in V_s} C(v_s)$ with $0 \leq |V_c| < |N^L(v_t)|$; here $C(v_s)$ is the set of nodes on the paths from $v_s$ to $v_t$ within $N^L(v_t)$; 
\item In addition, {\em one} of the following conditions holds: \textbf{(1)} $\M(G, v_t)$ = $\M(G_s, v_t)$, \ie $G_s$ is a ``factual'' explanation; or \textbf{(2)} $\M(G,v_t)$ $\neq$ $\M(G\setminus G_s, v_t)$, \ie 
$G_s$ is ``counterfactual''.
\ei 

Here, $G \setminus G_s$ is a graph obtained by removing the edge set $E_s$ from $G$,  while retaining all nodes. The inference process of $\M$ over $G_s$ (or $G \setminus G_s$) determines if $G_s$ is factual (or counterfactual). The input for inference is always a feature matrix $Z^0$ and an adjacency matrix, regardless of graph connectivity. 

\stitle{Remarks}. The above definition is well justified by established \gnn explainers, which generate (connected) subgraphs that are factual~\cite{ying2019gnnexplainer,yuan2021explainability}, counterfactual~\cite{bajaj2021robust}, or both~\cite{chen2024view}, as variants of the above definition. The main difference is that we explicitly distinguish explanatory nodes and connect sets. The added expressiveness allows us to track the varying set of ``influencing'' nodes and edges at different layers (as will be discussed in Section~\ref{sec:extended-explanations}). This is not fully addressed by prior \gnn explainers.
 
\vspace{.5ex}
The rest of the paper makes a case for \gcn-based node classification. For simplicity, we refer to ``class labels'' as ``labels'', and ``\gcn-based classifiers'' as ``\gcn''. Our approach {\em applies to other representative \gnns}, 
as verified in the Appendix~\ref{sec-exp}. 
\section{\slice: Layer-wise Explanations}
\label{sec-slice}

\subsection{Extending Explanations with Model-Slicing}  
\label{sec:extended-explanations}

We start with a notation of the model slice. 

\stitle{Model Slices}. To gain insights into the inner-working of \gnns and to enable a more granular interpretation, inspired by model slices~\citep{blouin2011modeling}, we create ``model slices'' of a trained \gnn $\M$ with selected individual layers. 

Given a \gnn $\M$ as a stack of an $L$-layered encoder $f_1$ and a predictor $f_2$, and a selected layer $l\in [1,L]$, an {\em $l$-sliced model} of $\M$, denoted as $\M^l$, is a \gnn as a stack of an encoder $f^l$ and 
the predictor $f_2$, where $f^l$ is a composite function. 
 \begin{equation}
    \label{quality}
        Z^l=f^l_1\circ f^{l-1}_1\circ \ldots \circ f^1_1(G)
\end{equation}

The inference process of an $l$-sliced model $\M^l$ computes a result consistently by (1) calling a composition $f^l_1\circ\ldots f^1_1(G)$ to generate output embeddings $Z^l$ up to layer $l$, and (2) deriving $z^l_{v_t}\in Z^l$ and applying  $f_2(z^l_{v_t})$ to obtain the result $\M^l(G,{v_t})$, for each node ${v_t}\in V_T$. In our abstraction, every function $f^l_1$ $(l\in[1,L])$ and $f_2$ uniformly takes a $|V|\times d$ embedding matrix as input and produces a matrix of the same size. In practice, we assume that the node features’ dimensionality $d$ remains unchanged across layers.

\begin{example}
Figure~\ref{fig-motivation} (upper part) depicts the architecture of an $l$-sliced model $\M^l$ derived 
from a given \gnn $\M$ with $L$ layers. The encoder of $\M^l$ includes the first $l$ \gnn layers $f^1_1, ..., f^l_1$. The predictor $f_2$ in $\M^l$ remains the same as in $\M$. The model parameters in the corresponding modules are shared between $\M$ and $\M^l$.
\label{slice_example}
\end{example}

\stitle{Explaining \gnns with Model Slicing}. 
We naturally extend \gnn explanations with 
model slicing.
Given a node $v_t\in V_T$ with an output $\M(G,v_t)$ to be explained, a connected subgraph $G^l_s$ of $G$ is an {\em explanation} of $\M(G,v_t)$ {\em at layer $l$}, if one of the following holds: 

\stab 
(1) The subgraph $G^l_s$ is a factual explanation of $\M(G,v_t)$, that is, $\M(G, v_t)$ = $\M(G^l_s, v_t)$; and also (b)  $\M(G, v_t)$ = $\M^l(G, v_t)$ = $\M^l(G^l_s,v_t)$. In this case,  we say that $G^l_s$ is a {\em factual explanation of $\M(G,v_t)$ at layer $l$}; or 

\stab 
(2) The subgraph $G^l_s$ is a counterfactual explanation of $\M(G,v_t)$, that is, $\M(G,v_t)$ $\neq\M(G\setminus G^l_s, v_t)$; and also (b) $\M(G,v_t)$ = $\M^l(G, v_t)$ $\neq\M^l(G\setminus G^l_s, v_t)$. For this case,  we say $G^l_s$ is a {\em counterfactual explanation of $\M(G,v_t)$ at layer $l$}.

Intuitively, we are interested in finding an explanation for an output $\M(G,v_t)$, which can be consistently derived by a sliced model at layer $l$, an ``earlier'' stage of inference of \gnn $\M$. 
We also provide a detailed discussion in the Appendix~\ref{sec-discuss}, where we show that this approach yields faithful explanations that align with the model’s true decision-making process.

\begin{figure}[tb!]
\centering
\centerline{\includegraphics[scale=0.19]{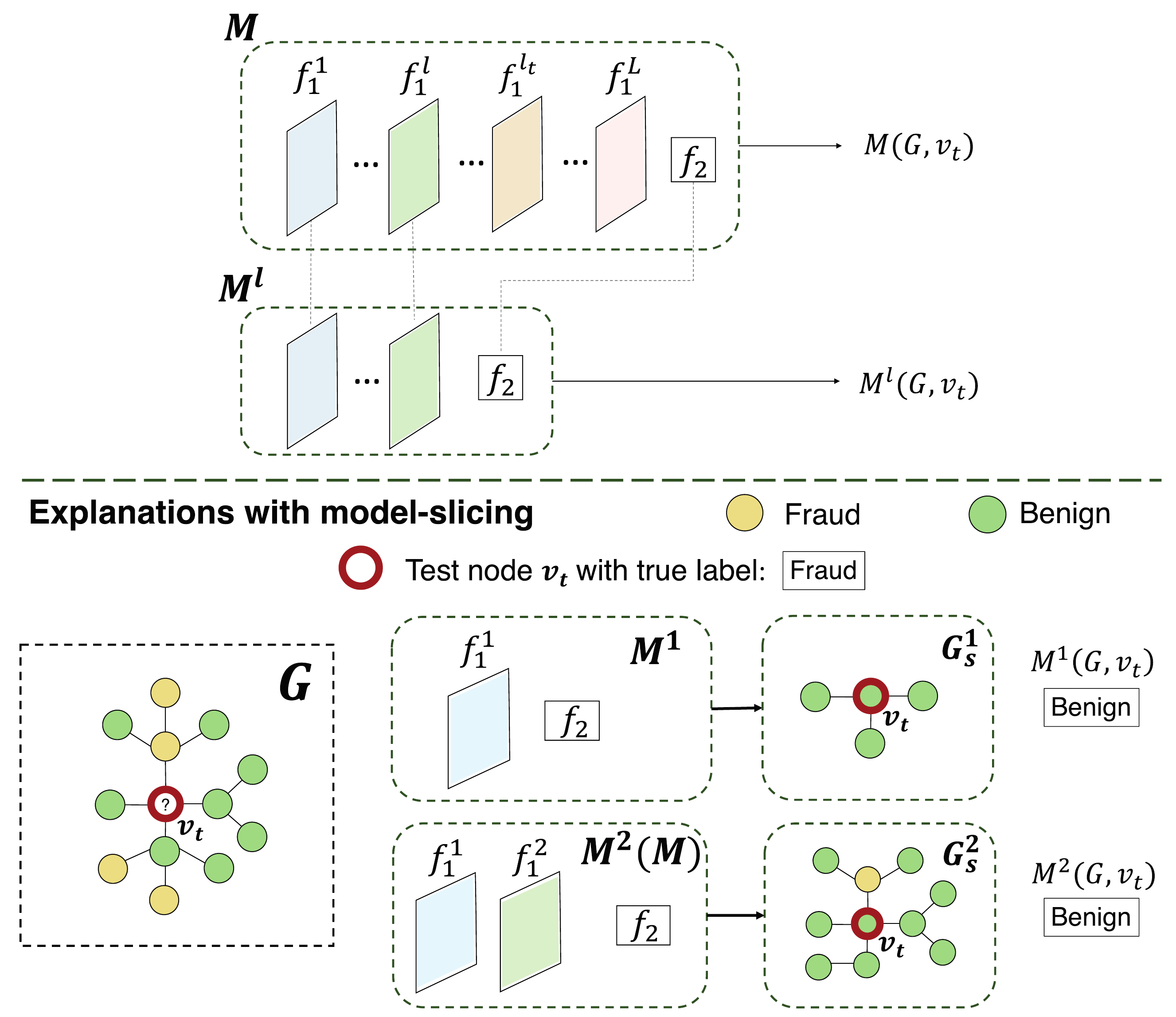}}
\vspace{-1ex}
\caption{Example of explanations with model-slicing. The upper part depicts the architecture of an $l$-sliced model $\M^l$, and the lower part shows factual explanations at each layer of $\M$ for diagnosing purposes.}
\label{fig-motivation}
\vspace{-2ex}
\end{figure}

\begin{example}
\label{example2}
Consider a two-layer \gnn-based node classifier $\M$ as in Example~\ref{slice_example} (lower part of Figure~\ref{fig-motivation}). The test node $v_t$ in graph $G=(V,E)$ is misclassified as ``Benign'' by $\mathcal{M}$, and its two ``sliced'' counterparts $\mathcal{M}^1$, and $\mathcal{M}^2$, at layer $1$ and layer $2$, respectively.  Here, $G^1_s$ and $G^2_s$ are factual explanations for the incorrect prediction at layers 1 and 2, respectively. While $G^2_s$ ensures consistency between $\mathcal{M}(G,v_t)$ and $\mathcal{M}(G^2_s,v_t)$ at the final layer, $G^1_s$ must maintain consistency in predictions for $v_t$ in both $\mathcal{M}$ and the sliced model $\mathcal{M}^1$. Specifically, $\mathcal{M}^1$ assigns $v_t$ the label ``Benign'' with $G^1_s$, aligning with $\mathcal{M}$'s final prediction. Inspecting the two explanations, $G^1_s$ and $G^2_s$, one can gain insights by observing the specific fraction of $G$ that may be responsible for an incorrect label ``Benign'' of node $v_t$. Better still, the explanations $G^1_s$ and $G^2_s$, putting together, suggest a ``chain of evidence'' that clarifies the misclassification of $v_t$ by $\M$ at early stage (layer 1) and its subsequent layers. 
\end{example}

\subsection{Explainability Measure}
\label{sec-quality}
While factual and counterfactual explanations are adopted to characterize ``relevant data'' that clarify \gnn outputs, it is necessary to quantify their explainability for specific diagnosing and interpreting needs. We adopt a novel bi-criteria measure below, in terms of label influence and embedding diversity, each with justification.

\stitle{Bi-criteria Explainability Measure}. 
\label{measure}
Given a graph $G=(V,E)$,  a \gnn $\M$ with $L$ layers, a specific layer number $l$ ($l\in [1,L]$), and a designated node $v_t$ for which the output $\M(G,v_t)$ needs to be explained, the {\em explainability} of an explanation subgraph $G^l_s$, with the explanatory node set $V^l_s$, at layer $l$ is defined as: 
    \begin{equation}
    \label{quality}
        f(G^l_s)=\gamma I(V^l_s)+(1-\gamma)D(V^l_s)
    \end{equation}
where (1) $I(V^l_s)$ quantifies the relative influence of $V^l_s$; (2) $D(G^l_s)$ measures the embedding diversity of $V^l_s$; and (3) $\gamma\in [0,1]$ is a balance factor. Note that the explainability $f(G^l_s)\in [0,1]$. 

We next introduce the functions $I(\cdot)$ and $D(\cdot)$.

\eetitle{Relative influence}. We extend \gnn feature propagation ~\cite{chen2024view,zhang2021grain} to layer-wise analysis. We introduce {\em relative influence}, a measure to evaluate the influence of a node $v$ at a specific layer to a node $u$ with output $\M(G,u)$. The relative influence of a node $v$ at layer $l$ on $u$ is quantified by the L1-norm of the expected Jacobian matrix: 
\begin{equation}
I_1(v,u) = \big|\big| \mathbf{E}[\frac{\partial Z_{v}^l}{\partial Z_{u}^0}] \big| \big|_1
\end{equation}

Further, we define the normalized relative influence as $I_2(v, u) = \frac{I_1(v,u)}{\sum_{w \in V} I_1(w,u)}$. We say that the node $u$ is {\em influenced} by $v$ at layer $l$, if $I_2(v,u) \ge h$, where $h$ is a threshold for influence. The Jacobian matrix can be approximately computed in \PTIME~\cite{chen2015measurement,xu2018representation}. 

Given an explanation $G^l_s$ with node set $V^l_s$, and $N^l(v)$ as $l$-hop neighbors of $v \in V^l_s$, the {\em relative influence set} of $v$ is $c(v) = \{u|I_2(v,u) \ge h, u\in N^l(v)\}$. Intuitively, it refers to the nodes that can be ``influenced'' by a node $v$ in the explanation $G^l_s$; the larger, the more influencing node $v$ is. 

We next introduce a node set function $I(V^l_s) = \frac{1}{|V|}\left|\bigcup_{v\in V^l_s}c(v)\right|$. Given an explanation $G^l_s$, it measures accumulated influence as the union of $c(v)$ for all $v\in V^l_s$. The larger, the more influential $G^l_s$ is. 

\eetitle{Embedding diversity}. Optimizing layer-wise explanations in terms of influence analysis alone may still lead to biased, ``one-sided'' explanations that tend to be originated from the similar type of most influential nodes across layers. We thus consider diversification of explanation generation to mitigate such bias, and enrich feature space~\cite{zhang2024ficom}. Given an explanation $G^l_s$ at layer $l$ of an output $\M^l(G,u)$ of a node $u$, we define the {\em diversity set} $r(v)$ as the nodes in $N^l(v)$ with a distance score $d$ above threshold $\theta$, \ie\ $r(v)=\{u|d(Z^l_{v}, Z^l_{u}) \geq \theta, u \in N^l(v)\}$. The distance function $d$ can be Euclidean distance of node embeddings~\cite{li2019approximate}, and $D(V^l_s)$ quantifies diversity as the union of $r(v)$ for all $v\in V^l_s$: $D(V^l_s) = \frac{1}{|V|}\left|\bigcup_{v\in V^l_s}r(v)\right|$.

Putting these together, we favor explanation with better explainability, which can contribute (1) higher accumulated influence relative to the result of interests from current layers, and (2) meanwhile reveals more diverse nodes on its own, in terms of node embedding. Our method is configurable: other influence and diversity (\eg mutual information~\cite{ying2019gnnexplainer}) or explanability measures (\eg fidelity~\cite{funke2022zorro}) can be used to generate explanations for task-specific needs. To demonstrate the poor performance of relying on a single metric, we conducted a detailed ablation study in our experimental section.

\begin{figure}[tb!]
\centering
\centerline{\includegraphics[scale=0.25]{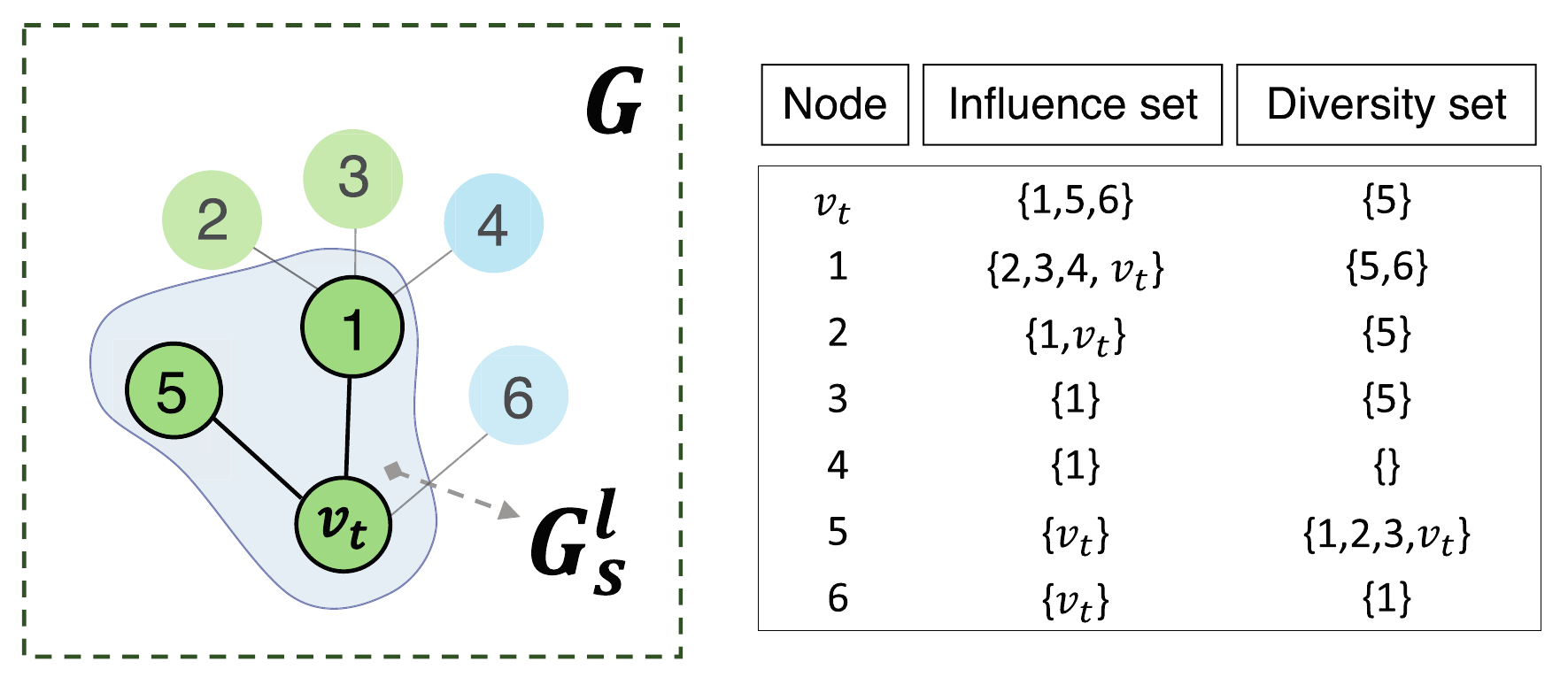}}
\vspace{-3mm}
\caption{Explanation with high influence and high embedding diversity (k=3)}
\label{qualityexample}
\vspace{-4mm}
\end{figure}

\begin{example}
Figure ~\ref{qualityexample} showcases an explanation with high explainability for $v_t$. The influence set and diversity set of the explanation $G^l_s$ with node set $V^l_s=\{1,5,v_t\}$ is $\{1,2,3,4,5,6,v_t\}$ and $\{1,2,3,5,6,v_t\}$. Among all the possible explanations of $G$ size bounded by 3, $G_s$ is favored due to the maximum quality score calculated in Equation~\ref{quality}.
\end{example}

\subsection{Explanation Generation Problem}
\label{sec-problem}

\stitle{Configuration}. 
We next formulate the problem of explanation generation with model slicing. We characterize the explanation request as a {\em configuration} $\C$ = $(G,\M, v_t, \L, l_t, k)$, as follows:
\begin{itemize} 
\item a graph $G$, a \gnn $M$ with $L$ layers, where the $l$-th layer is indexed with a layer number $l$;  
\item  a {\em target layer} $l_t$ and a {\em target node} $v_t$, to specify a designated output $\M^{l_t}(G,v_t)$ to be explained; 
\item  {\em source layers} $\L$, a set of specified layer numbers, such that for each layer $l\in\L$ $(l\in[1,L])$, an $l$-sliced model $\M^l$ is derived to provide explanations for $\M^{l_t}(G,v_t)$;  
\item $k$ is a size bound w.r.t. the number of explanatory nodes.  
\end{itemize}

Given a target result $\M^{l_t}(G,v_t)$ to be explained, the explanation generation with model-slicing is to generate, at each layer $l\in\L$, an optimal explanation $G^{*l}_s$ of $G$ with an explanatory node set $V_s$, for the output $\M^{l_t}(G,v_t)$, such that 
\begin{equation}
G^{*l}_s = \argmax_{|V_s|\leq k} f(G^l_s)
\end{equation}

The configuration supports progressive analysis of \gnn outputs:  (1) For ``progressive'' diagnosis, answering why $\M$ assigns a wrong label $\M(G,v_t)$ (Example~\ref{example2}), set $\C$ as $(G, \M, v_t, \{1,\ldots,L\}, L, k)$, where $L$ is the output layer. (2) To analyze $\M$'s behavior at selected layers for target layer $l_t$, set $\C$ as $(G, \M, v_t, \L, l_t, k)$. (3) Conventional \gnn explanation is a special case with $\L = \{L\}$ and $l_t = L$.

\stitle{Complexity}. We begin hardness analysis by showing that it is feasible to check if a subgraph is an explanation.

\begin{lem}
\label{lm-verify}
Given a configuration $\C$ = $(G,\M, v_t, \L, l_t, k)$, and a subgraph $G_s$ of $G$, it is in \PTIME to verify if $G_s$ is an explanation \wrt $\M^{l_t}(G,v_t)$ at layer $l$.
\end{lem}

\begin{proofS}
We consider a configuration \( C \) for a fixed GNN model \( \M^{l_t} \) at layer \( l_t \). As shown in cost analyses for GNNs, \( \M^{l_t} \) is constant with the same architecture and weights, ensuring deterministic inference in PTIME. To verify \( G_s \), we evaluate a Boolean query for "factual" or "counterfactual" conditions at layer \( l_t \), involving limited calls to \( \M \) and equality tests, like \( \M(G,v_t) = \M^{l_t}(G_s, v_t) \), all within \PTIME.
\end{proofS}

\vspace{-2mm}
\begin{theor} 
\label{hardness}
The decision problem of Explanation generation with model slicing is \NP-hard.
\end{theor}

\vspace{-1mm}
\begin{proofS}
We reduce the maximum coverage problem~\cite{papadimitriou1988optimization} to our problem by constructing a graph \(G\) where nodes represent sets and elements, and edges connect elements to their sets and sets to a target node \(v_t\). A 2-layer GNN \(\M\) is trained to label \(v_t\) correctly. A size-\(k\) factual explanation for \(v_t\) exists if and only if there is a solution to the maximum coverage problem with \(|\bigcup_{S \in \S'} S| \geq b\). Since the decision problem of maximum coverage problem is NP-hard, our problem is also NP-hard.
\end{proofS}

We next introduce novel and efficient algorithms for layer-wise explanation generation, with guarantees on the quality. 

\begin{algorithm}[tb!]
    \renewcommand{\algorithmicrequire}
    {\textbf{Input:}}
    \renewcommand{\algorithmicensure}
    {\textbf{Output:}}
    \caption{:~\slicess (single node, single source layer)}
    \begin{algorithmic}[1]
        \REQUIRE a configuration $\C$ = $(G,\M, v_t, \{l\}, l_t, k)$;  
        \ENSURE an explanation $G^l_s$ \wrt 
        $\M^{l_t}(G,v_t)$ at \gnn layer $l$. 
        \STATE initialize $G^l_{v_t}=(V^l_{v_t}, E^l_{v_t})$ as the $l$-hop subgraph in $G$
        \STATE set $V_s := \emptyset$; $V_c := \emptyset$
        \WHILE{$|V_s| < k$ and $V_{v_t}^l \backslash V_s\neq\emptyset$}
            \STATE expand $V_s$ by pair-wise node selection in $V^l_{v_t}$
        \ENDWHILE
        \STATE select $V_c\in V^l_{v_t}$ to keep induced subgraph connected 
        \STATE $V_s'=V_s\cup V_c$
        \STATE set $G^l_{s}$ as the subgraph derived from node set $V_s'$; 
        \IF{\kw{Verify}($G^l_{s}$, $C$) = true} 
        \RETURN $G^l_s$; 
        \ENDIF
        \WHILE{\kw{Verify}($G^l_{s}$, $C$) = false and $V^l_{v_t}\backslash V_s'\neq \emptyset$}
            \STATE replace the lowest-scoring node in $V_s$ with highest-priority node in $V^l_{v_t}\backslash V_s'$ while ensuring connectivity by BFS
            \STATE set $G^l_s$ as the subgraph derived from node set $V_s'$;
           \IF{\kw{Verify}($G^l_s$, $C$) = true}  \RETURN  \ttrevise{$G^l_s$}; 
        \ENDIF
        \ENDWHILE
        \RETURN $G^l_{v_t}$.
    \end{algorithmic}
    \label{alg:greedy}
\end{algorithm}

\label{sec-algorithm}

\vspace{-1mm}
\subsection{Generating Model-slicing Explanation}
Our first algorithm generates layer-wise  explanations for a single node at a designated target layer.
\label{sec-quality}

\begin{theor}
\label{theor-approx}
Given a configuration $\C$ that specifies a single node $v_t$ and a single source layer $l$, there is a $\frac{1}{2}$-approximation for the problem of explanation generation with model-slicing,  which takes $O(k|G^l_{v_t}|\log|V^l_{v_t}|+L|V^l_{v_t}|^2)$ time.
\end{theor}

We next introduce such an algorithm as a constructive proof.

\begin{figure}[tb!]
\vspace{-2mm}
\centering
\centerline{\includegraphics[scale=0.2]{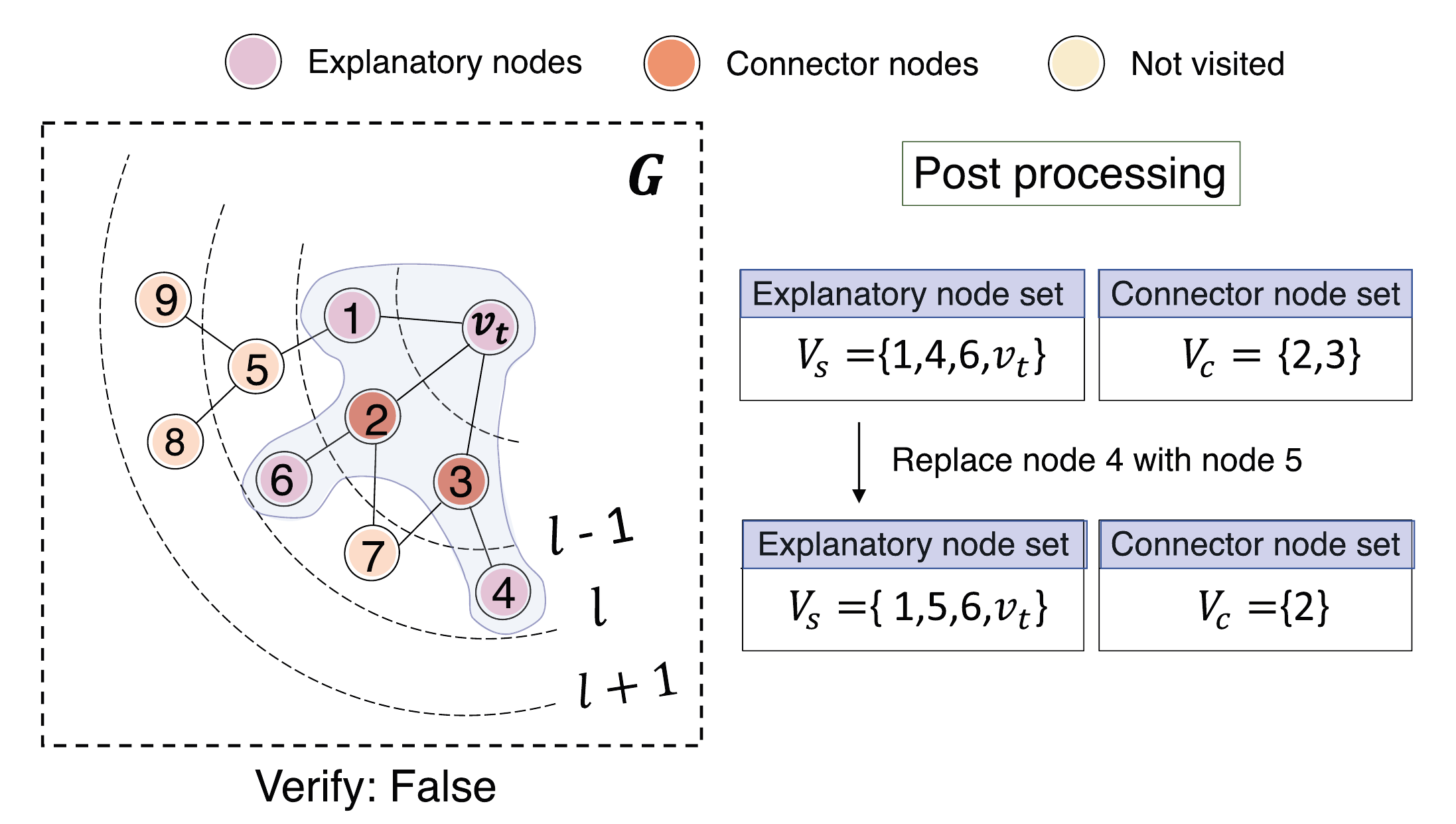}}
\vspace{-1mm}
\caption{A running example of SliceGX (k=4)}
\label{runningexample}
\vspace{-3mm}
\end{figure}

\stitle{Algorithm}. The algorithm, denoted as  
\slicess and illustrated as Algorithm \ref{alg:greedy}, takes a configuration $\C$ as input, and generates an explanation $G^l_s$ for the output $\M^{l_t}(G,v_t)$ 
of a single test node $v_t$, at a specific layer $l$. 

\eetitle{Initialization} (lines 1-2). 
Algorithm~\slicess leverages data locality in \gnn inference: For node $v_t$, an $l$-hop \gnn computes $\M(G,v_t)$ by visiting $l$-hop neighbors via message passing. Prior work treats the $l$-hop subgraph as the computational graph, where $l$ is the number of layers~\cite{luo2020parameterized}. For layer $l$ ($l\in [1,L]$), \slicess initializes $G^l_{v_t}$ as the $l$-hop subgraph centered at test node $v_t$, induced by nodes and edges within $l$-hop of $v_t$, to refine the search space.

\eetitle{Generation Phase} (lines 3-9). \slicess then applies a greedy strategy to select a set of explanatory nodes \ttrevise{$V_s$} by iteratively choosing a pair of nodes that maximize a marginal gain of $f(\cdot)$ (line 4). It then induces an explanation with the connect sets by definition (Section \ref{sec:preliminaries}). Finally it invokes a procedure \textsf{Verify} to check if \ttrevise{$G^l_s$} satisfies constraints to be an explanation and return \ttrevise{$G^l_s$} if so.

\eetitle{Post-processing} (lines 10-15). 
The generation phase does not guarantee that $G^l_s$ passes \textsf{Verify}. Thus, a lightweight post-processing phase is added. "Back up" nodes are sorted by explainability scores, replacing the lowest-scoring node in $V_s$ while preserving connectivity. The \textsf{Verify} procedure checks if the subgraph yields factual or counterfactual explanations and returns it if valid.

Finally, if no explanation is identified, \slicess returns $G^l_{v_t}$, the $l$-hop subgraph of $v_t$, a factual explanation of $\M(G, v_t)$ by default.

\stitle{Procedure \verify}. 
The procedure verifies if the subgraph $G^l_s$ induced by $V_s \cup V_c$ is a valid explanation, by determining if it is a factual or counterfactual explanation at layer $l$. Following Lemma~\ref{lm-verify}, the verification is in \PTIME. 

\begin{example}
Figure~\ref{runningexample} illustrates an explanation generated by \slicess. The explanatory node set \(V_s = \{1, 4, 6, v_t\}\) is initially chosen to maximize the explainability score with the desired size $k$. To ensure connectivity, connector nodes \(V_c = \{2, 3\}\) are added, but the subgraph fails verification. So, node 4 is replaced by node 5, updating \(V_c\) to \(\{2\}\), resulting in a valid explanation with \(V_s = \{1, 5, 6, v_t\}\) and \(V_c = \{2\}\).
\end{example}


\stitle{Correctness and Approximability}  
Algorithm~\slicess ensures correctness by selecting $V_s$ and inducing their connect set, such that $G^l_s$ is either a factual or counterfactual explanation at layer $l$. For approximability, we present the following detailed analysis:

\begin{lem}  
\label{lm-submodular}  
Given an output $ \M(G, v_t) $ and an explanation $ G^l_s $ at layer $ l $, the explainability measure $ f(G^l_s) $ is a monotone submodular function for the node set of $ G^l_s $.  
\end{lem}

The bi-criteria explainability measure \( f(\cdot) \) is a submodular function maximizing relative influence \( I(V^l_s) \) and embedding diversity \( D(V^l_s) \). We then employ the approximation-preserving reduction to a Max-Sum Diversification problem~\cite{gollapudi2009axiomatic} which indicates a greedy algorithm offering a \( \frac{1}{2} \)-approximation for this goal. We present the detailed proof in the Appendix~\ref{sec-appendix}.


\stitle{Time Cost}. \slicess takes  $O(|G^l_{v_t}|)$ time to generate $l$-hop subgraph centered at $v_t$ with breadth-first traversal (line~1). It requires $\lceil \frac{k}{2}\rceil$ iterations of explanatory nodes selection, where each iteration takes $O(T|V^l_{v_t}|)$ to identify the nodes with the maximum marginal gain (based on Equation~\ref{quality}). Here, $T$ is the average time for calculating the marginal gain for each node $v\in V^l_{v_t}$. Next, it requires time cost in $O(|G^l_{v_t}|*log|V^l_{v_t}|)$ to induce connect sets and then to induce a connected subgraph to be further verified. Procedure \kw{Verify} checks whether the subgraph $G^l_s$ is an explanation in $O(l|G^l_{v_t}|)$ time. The post processing takes at most $|V^l_{v_t}|$ rounds and a verification cost in $O(l|V^l_{v_t}|)$ time. Combining these components, as $T$ and $O(l|G^l_{v_t}|)$ are relatively small compared to the dominant terms, the overall time cost is in $O(k|G^l_{v_t}|\log|V^l_{v_t}|+l|V^l_{v_t}|^2)$. Here, $l|V^l_{v_t}|^2$ comes from the post-processing stage (line 10-15, Alg 1). Empirically, such worst-case scenarios occur in only about 3\% of instances, so the actual cost is much lower in practice.

\vspace{.5ex}
With the above analysis, Theorem~\ref{theor-approx} follows. 

As \slicess method generates explanations sequentially for each  test node, it is computationally expensive, especially when dealing with multiple test nodes and layers.  To address this limitation, we introduce two optimization techniques, \slicems and \slicemm, to incrementally generate explanations. Specifically, \slicems updates auxiliary structures for all test nodes $V_T$ during the node selection phase. This optimization  reduces redundant computations and enhances scalability for large-scale scenarios. Building on \slicems, \slicemm  further extends the approach to handle multi-layer scenarios by introducing a “hop jumping” strategy that tracks  explanatory node sets across multiple layers. This enables  \slicemm to identify explanations that are both efficient and comprehensive across different layers of the network. Due to the limited  space, we provide the details in the Appendix~\ref{sec-ms}.


\begin{table*}[tb!] 
\centering
\caption{
    Fidelity Evaluation of GNN Explainers. 
    Fidelity+ (Fid+, higher is better $\uparrow$) and Fidelity- (Fid-, lower is better $\downarrow$). 
    The best and second-best performances are denoted in \textbf{bold} and \underline{underlined} fonts, respectively.
    `---' indicates that the method failed to complete within the 24-hour time limit.
}
\label{tab:fidelity_evaluation}
\vspace{-3mm}
\begin{tabularx}{\textwidth}{l *{6}{Cc}}
\toprule
& \multicolumn{2}{c}{BA-shapes} & \multicolumn{2}{c}{Tree-Cycles} & \multicolumn{2}{c}{Cora} & \multicolumn{2}{c}{Coauthor CS} & \multicolumn{2}{c}{YelpChi} & \multicolumn{2}{c}{Amazon} \\
\cmidrule(lr){2-3} \cmidrule(lr){4-5} \cmidrule(lr){6-7} \cmidrule(lr){8-9} \cmidrule(lr){10-11} \cmidrule(lr){12-13}
Explainers & Fid+ ($\uparrow$) & Fid- ($\downarrow$) & Fid+ ($\uparrow$) & Fid- ($\downarrow$) & Fid+ ($\uparrow$) & Fid- ($\downarrow$) & Fid+  ($\uparrow$) & Fid- ($\downarrow$) & Fid+ ($\uparrow$) & Fid- ($\downarrow$) & Fid+ ($\uparrow$) & Fid- ($\downarrow$) \\
\midrule
GNNExplainer & 0.2665 & 0.7264 & 0.6065 & 0.6611 & 0.1088 & 0.8173 & 0.4302 & 0.1525 & 0.2041 & 0.4238 & --- & --- \\
PGExplainer & 0.2763 & 0.6697 & 0.3689 & 0.2551 & 0.2739 & 0.2041 & 0.3217 & 0.2815 & \underline{0.3482} & 0.4255 & \underline{0.0421} & 0.8469 \\
GraphMask & 0.4302 & \underline{0.3580} & 0.1243 & \underline{0.0280} & 0.1260 & 0.7202 & 0.2341 & 0.4512 & 0.1245 & 0.7241 & 0.0012 & \underline{0.7187} \\
SubgraphX & 0.2307 & 0.4018 & \underline{0.7012} & 0.0782 & 0.4748 & \underline{0.1003} & 0.5142 & 0.1451 & --- & --- & --- & --- \\
SAME & 0.0303 & 0.6662 & 0.5124 & 0.4920 & 0.3984 & 0.2522 & 0.3516 & \underline{0.1242} & --- & --- & --- & --- \\
FlowX & \underline{0.4682} & 0.7062 & 0.5052 & 0.2583 & \underline{0.6430} & 0.7471 & \underline{0.5189} & 0.1432 & 0.3014 & \underline{0.2414} & --- & --- \\
\textbf{SliceGX} & \textbf{0.6918} & \textbf{0.0670} & \textbf{0.8047} & \textbf{0} & \textbf{0.7117} & \textbf{0.0531} & \textbf{0.7003} & \textbf{0.0341} & \textbf{0.5014} & \textbf{0.1294} & \textbf{0.1724} & \textbf{0.4521} \\
\bottomrule
\end{tabularx}
\end{table*}

\begin{figure*}[tb!]
    \vspace{-3mm}
    \centering
    \begin{minipage}[]{0.48\linewidth}
        \centering
        \includegraphics[width=\linewidth]{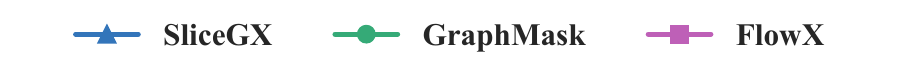}
    \end{minipage}
    \hspace{0mm}  
    \begin{minipage}[]{0.48\linewidth}
        \centering
        \includegraphics[width=\linewidth]{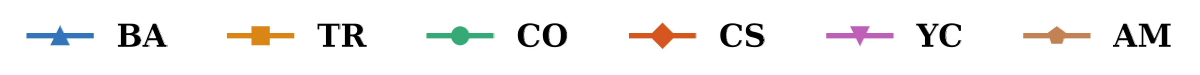}
    \end{minipage}
\end{figure*}

\begin{figure*}[tb!]
	\vspace{-7mm}
	\centering
    \subfigure[Fidelity+: Varying source layers]
	{\includegraphics[width=0.24\linewidth]{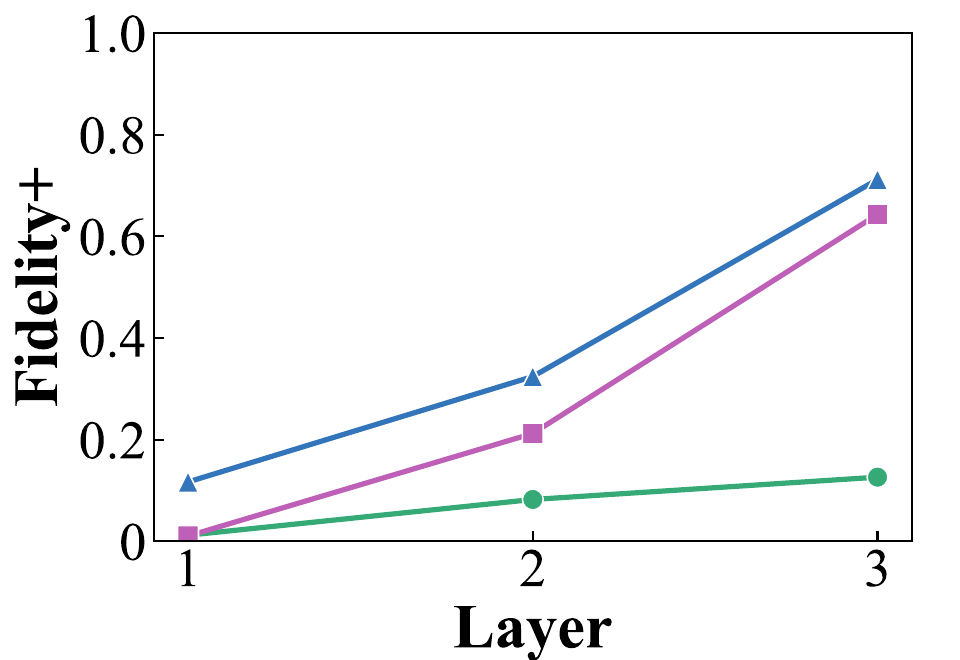}
		\label{fig:layer_wise+}}
    \subfigure[Fidelity-: Varying source layers]
	{\includegraphics[width=0.24\linewidth]{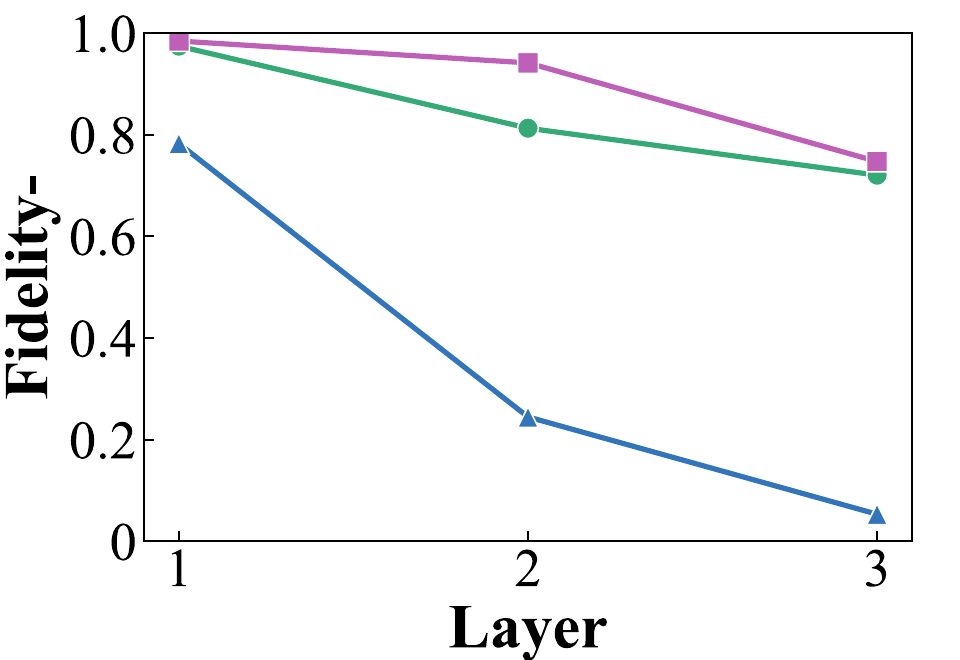}
	    \label{fig:layer_wise-}}
    \subfigure[Fidelity+: Varying $k$ (\% of Nodes)]
	{\includegraphics[width=0.24\linewidth]{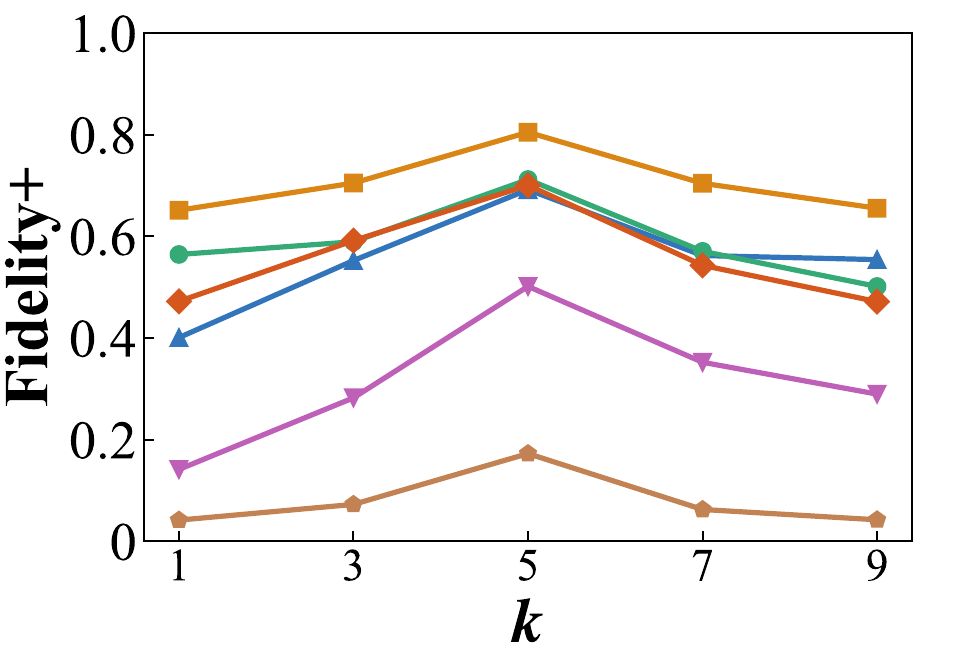}
		\label{fig:k+}}	
    \subfigure[Fidelity-: Varying $k$ (\% of Nodes)]
	{\includegraphics[width=0.24\linewidth]{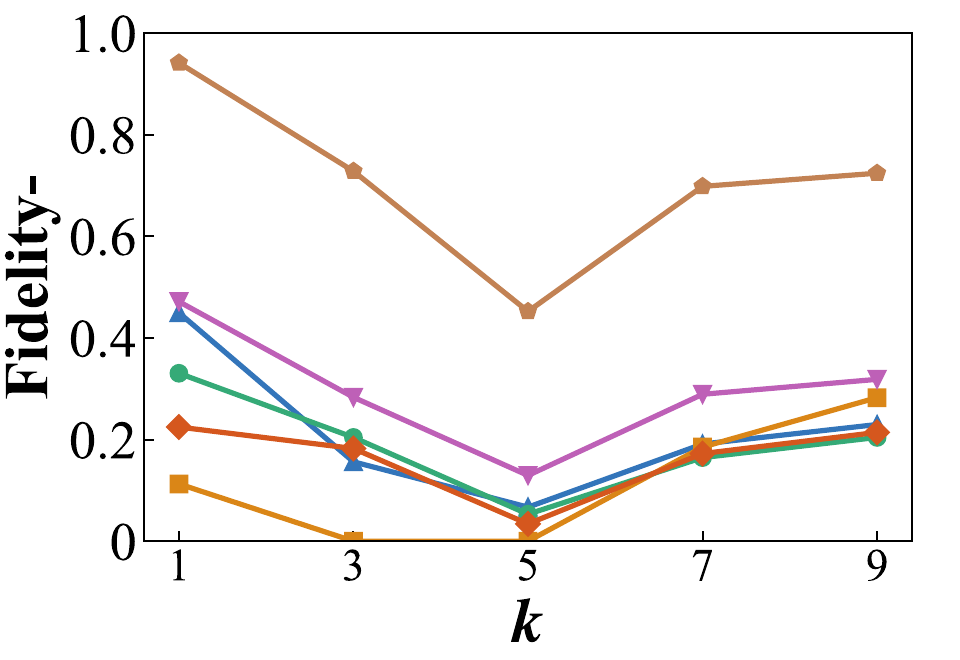}
		\label{fig:k-}}
\end{figure*}

\begin{figure*}[tb!]
    \vspace{-6mm}
    \centering
    \begin{minipage}[]{0.48\linewidth}
        \centering
        \includegraphics[width=\linewidth]{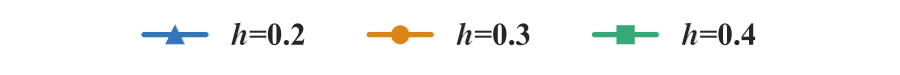}
    \end{minipage}
    \hspace{0mm}  
    \begin{minipage}[]{0.48\linewidth}
        \centering
        \includegraphics[width=\linewidth]{figures/legend1cd.pdf}
    \end{minipage}
\end{figure*}

\begin{figure*}[tb!]
    \vspace{-7mm}
	\centering
	\subfigure[Fidelity+: Varying $(h,\theta)$]
	{\includegraphics[width=0.24\linewidth]{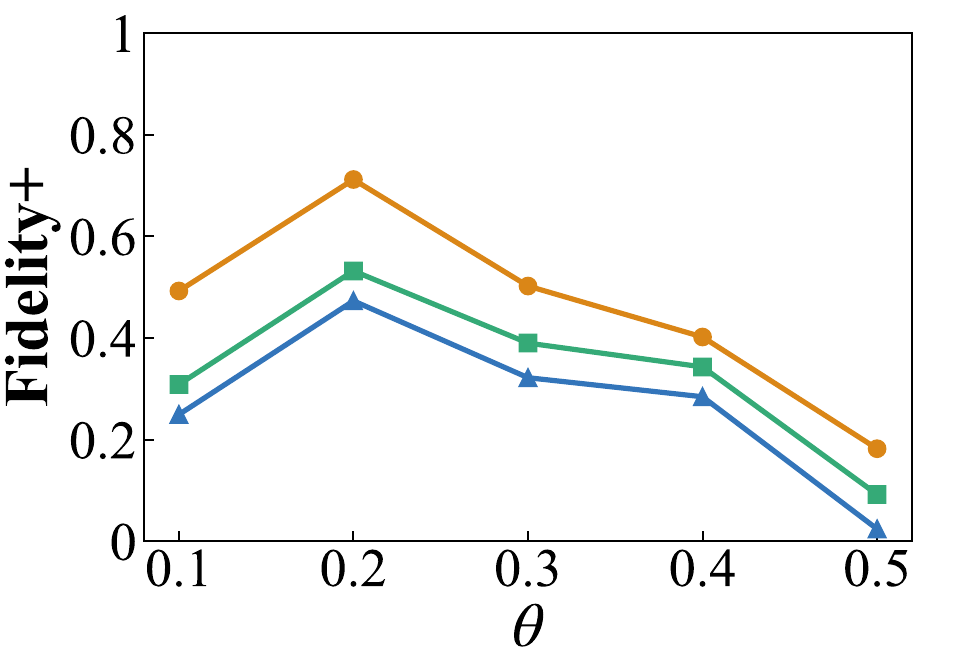}
		\label{fig:h+}}
    \subfigure[Fidelity-: Varying $(h,\theta)$]
	{\includegraphics[width=0.24\linewidth]{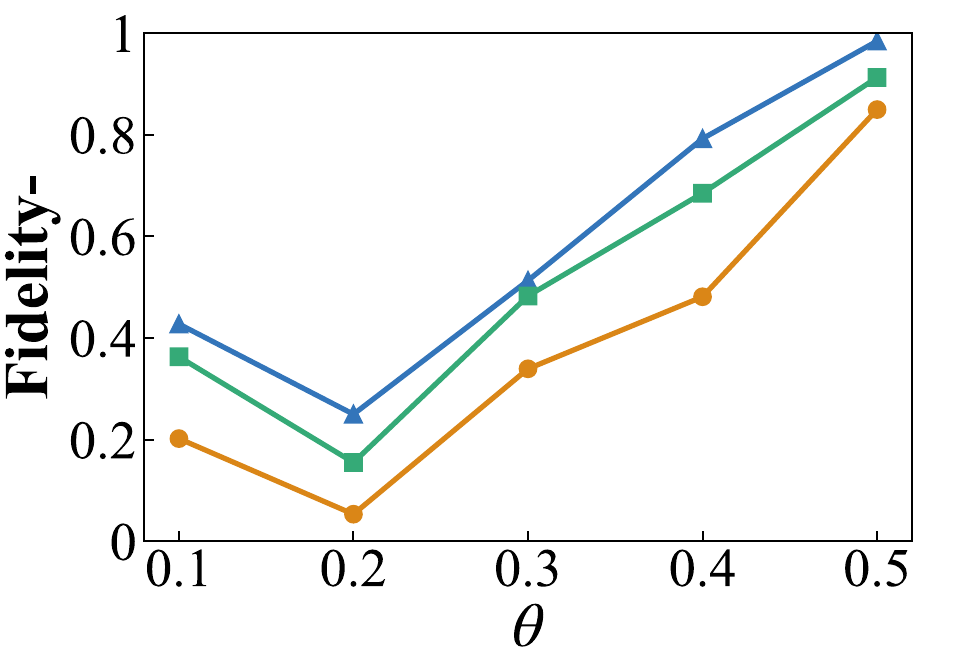}
	    \label{fig:h-}}
    \subfigure[Fidelity+: Varying $\gamma$] 
	{\includegraphics[width=0.24\linewidth]{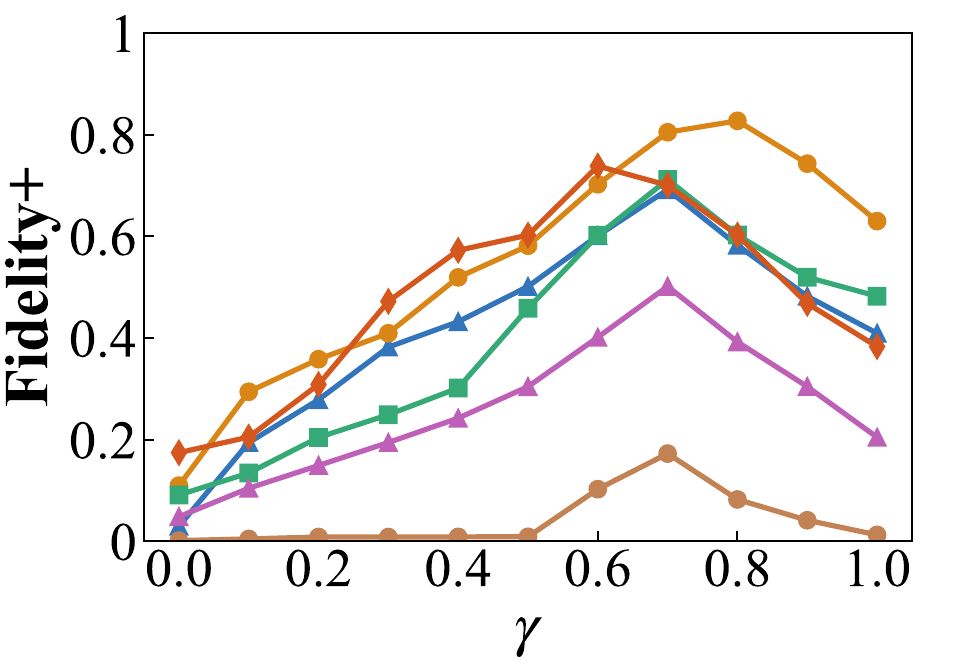}
		\label{fig:gamma+}}	
    \subfigure[Fidelity-: Varying $\gamma$] 
	{\includegraphics[width=0.24\linewidth]{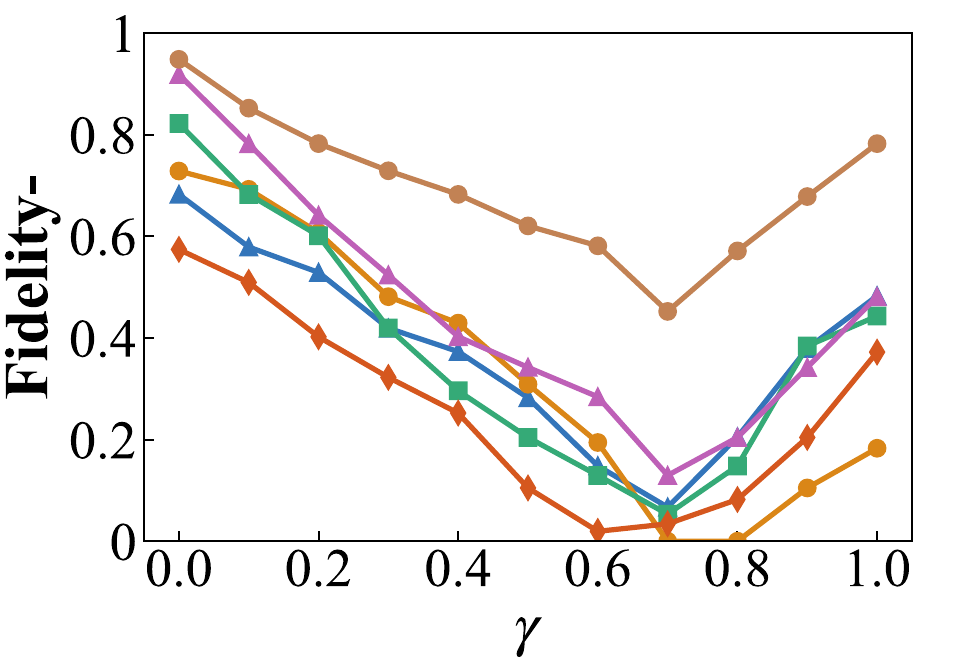}
		\label{fig:gamma-}}
    \vspace{-3mm}
	\caption{Impact of factors on quality of explanations}
    \vspace{-3mm}
	\label{fig:quality-all}
\end{figure*}



\section{Experimental Study}
\label{sec:exp}

To evaluate \slice-based explanation algorithms, We investigate the following questions. 
\textbf{RQ1}: Can \slice-based algorithms generate high-quality explanations for designated layers? 
\textbf{RQ2}: What are the impact of important factors to the quality and efficiency of explanation generation? 
\textbf{RQ3}: How fast can \slice generate explanations in a progressive manner? We also illustrate \slice-based explanation for model debugging. {\bf Our code is given at~\cite{Code}.}

\vspace{-1ex}
\subsection{Experimental Setup}
\vspace{-2ex}

\stitle{Datasets}. 
We adopt six benchmark graph datasets, as summarized in Table~\ref{tab-data}: \textbf{BA-shapes (BA)~\cite{fey2019fast}}, \textbf{Tree-Cycles (TR)~\cite{fey2019fast}}, \textbf{Cora (CO)~\cite{kipf2016semi}}, \textbf{Coauthor CS (CS)~\cite{shchur2018pitfalls}}, \textbf{YelpChi (YC)~\cite{dou2020enhancing}}, and \textbf{Amazon (AM)~\cite{zeng2019graphsaint}}. These datasets cover synthetic and real-world graphs for tasks such as node classification, spam detection, and co-authorship analysis. We provide the details in the Appendix~\ref{sec-ds}.

\stitle{GNN classifiers.} For all datasets, we have trained 3-layer Graph Convolutional Networks (\gcns)~\cite{kipf2016semi} for comparison. We used the Adam optimizer~\cite{diederik2014adam} with a learning rate of 0.001 for 2000 epochs.
We also showcase that \slice can be applied to other types of \gnns in the Appendix~\ref{sec-exp}. 

\stitle{GNN explainers.} 
We have implemented the following explainers:  \textbf{(1)} our \slice method based on  \slicess.  \textbf{(2) \gnnexp}~\cite{ying2019gnnexplainer} uses mutual information to identify key edges and node features for instance-level classification. \textbf{(3) \subx}~\cite{yuan2021explainability} applies Monte Carlo tree search to find important subgraphs via node pruning for graph classification. \textbf{(4) \pgexp}~\cite{luo2020parameterized} models edges as conditionally independent Bernoulli variables, optimized for mutual information. \textbf{(5) \gmask}~\cite{schlichtkrull2020interpreting} trains edge masks to prune irrelevant edges and generate layer-specific subgraphs. \textbf{(6) \same}~\cite{ye2023same} uses two-phase expansion MCTS to efficiently find optimal multipiece explanations via k-hop Shapley. \textbf{(7) \flowx}~\cite{gui2023flowx} uses a sampling scheme to compute Shapley values, which quantify the importance of "message flows" to the model's prediction. We also provide the results for \slicems and \slicemm in the Appendix~\ref{sec-exp}.

\stitle{Evaluation metrics.} We extend Fidelity~\cite{kakkad2023survey, prado2024survey} for layer-wise explanations. It measures faithfulness of explanations at layer $l$ with the normalized difference between the output of a \gnn over $G$ and $G \setminus G^l_s$, where $G^l_s$ is the explanation at layer $l$. 

\sstab (1) We define \textbf{Fidelity+} as:
\begin{small}
    \begin{equation}
        Fidelity+ = \frac{1}{|V_T|} \sum_{v_t \in V_T} (\phi^l(G, v_t)_c - \phi^l(G \setminus G^l_s, v_t)_c)
    \end{equation}

\end{small}
where $\phi^l(\cdot)_c$ is the predicted likelihood of label $c$ for $v_t$ from $\M^l$, and $c$ is the label assigned by $\M^l(G, v_t)$. Fidelity+ quantifies the impact of removing the explanation on label assignment, with higher values indicating better explanations. 

\sstab 
(2) Similarly, we define \textbf{Fidelity-} of $G^l_s$ as:
\begin{small}
    \begin{equation}
        Fidelity-=\frac{1}{|V_T|}\sum_{v_t\in V_T}{(\phi^l(G,v_t)_c-\phi^l(G^l_s,v_t)_c)}
    \end{equation}
\end{small}
Fidelity- at layer $l$ quantifies how ``consistent'' the result of \gnn is over $G^l_s$ compared with its result at layer $l$ over $G$. The smaller the value is, the better. 


We also report the time cost of explanation generation. For learning-based explainers, the cost includes learning overhead. 

\vspace{-0.5ex}
\subsection{Experimental Results}
\label{experiment_result}

\vspace{-1.5ex}
\stitle{Exp-1: Quality of explanations.} 
We evaluated explanation quality for target layer output, setting source layer $\L = \{3\}$, target layer $l_t = 3$. The evaluation was conducted on a test set of 100 sampled, labeled nodes ($|V_T|=100$), with the explanation size $k$ set to 5\% of the total nodes in each graph.

\eetitle{Overall Fidelity}. We conducted a quantitative fidelity evaluation, with the results summarized in Table \ref{tab:fidelity_evaluation}. Our method, \slice, consistently outperforms baseline explainers across all datasets, achieving the best performance in both Fidelity+ and Fidelity- metrics. Notably, \slice obtains a Fidelity- score of 0 on the Tree-Cycles dataset—a theoretically optimal result. This demonstrates a substantial improvement over existing state-of-the-art \gnn explainers.

\eetitle{Impact of source layers}. We varied source layers from 1 to 3 (Figs.~\ref{fig:layer_wise+}, \ref{fig:layer_wise-}) to assess robustness with a fixed target layer. Here, we selected \gmask and \flowx as the baseline explainers for comparison, as other existing explainers lack support for intermediate layer explanations. 
For both Fidelity+ and Fidelity-, \slice consistently outperforms both \gmask and \flowx across all layers. This demonstrates \slice's superior ability to generate robust and high-quality layer-wise explanations

\eetitle{Impact of $k$}. 
The variation of $k$ from 1\% to 9\% of nodes in each graph was employed in the analysis. Shown in  Figs.~\ref{fig:k+} and~\ref{fig:k-}, the best fidelity performance across all datasets was achieved at k=5\%. Thus we selected k=5\% as the optimal configuration, as it provides the strongest explanator validity while maintaining sparsity.

\eetitle{Impact of ($h$, $\theta$)}. We evaluated thresholds $h$ (influence) and $\theta$ (embedding diversity) on the Cora dataset. Figs.~\ref{fig:h+} and~\ref{fig:h-} show: (1) Proper embedding diversity ($\theta$ $\in [0.2, 0.3]$) improves explanations by mitigating bias from similar nodes, but excessive diversity introduces noise and reduces quality. (2) Higher influence ($h$ $\in [0.3, 0.4]$) improves explanations, but overly influential nodes can introduce bias if their embeddings differ significantly from target nodes.

\eetitle{Impact of $\gamma$}. 
We also evaluated the balance factor $\gamma$ between relative influence and embedding diversity in Equation~\ref{quality}  across all datasets. As shown in Figs.~\ref{fig:gamma+} and~\ref{fig:gamma-}, when $\gamma$ is too small (\ie $\gamma < 0.3$), the objective favors embedding diversity, potentially admitting noisy nodes and degrading explanation quality. Conversely, when $\gamma$ is too large (\ie $\gamma > 0.8$), The objective overemphasizes influence, biasing explanations toward influential yet redundant nodes.
\slice achieves the best performance when $\gamma$ lies in the moderate range (\ie $\gamma \in [0.6, 0.8]$), effectively balancing influence and diversity to yield high-quality explanations. These results validate our bi-criteria objective, demonstrating that \slice adapts to diverse explanation needs by tuning $\gamma$.

\begin{figure}[tb!]
    \centering
    
    \begin{minipage}[]{1\linewidth}
        \centering
        \includegraphics[width=\linewidth]{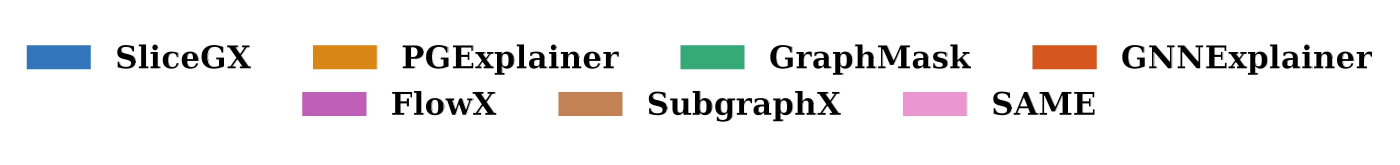}
    \end{minipage}
    
    \vspace{-3mm} 
    
    \subfigure[Efficiency: Overall]
    {\includegraphics[width=0.47\linewidth]{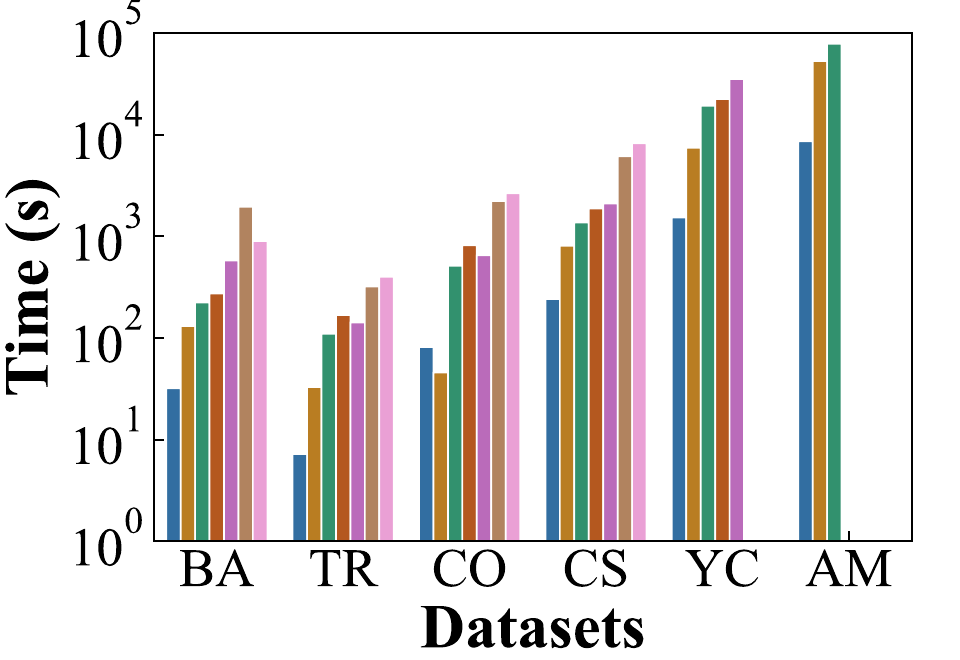}
        \label{fig:time}}
    \hspace{-1mm}
    \subfigure[Efficiency: Varying source layer]
    {\includegraphics[width=0.47\linewidth]{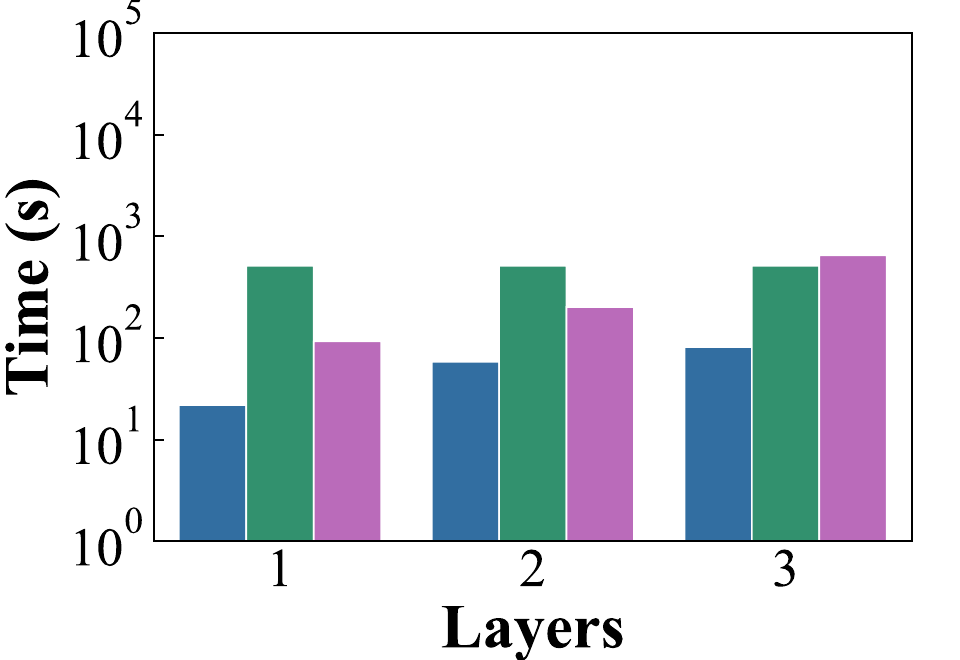}
        \label{fig:layer_time}}
    
    \vspace{-3mm}
    
    \caption{Efficiency of explanation generation}
    \label{fig:slice}
    \vspace{-3ex}
\end{figure}

\stitle{Exp-2: Efficiency of explanation generation}. We next report the efficiency of \slice, compared with other explainers.

\eetitle{Overall performance}. Figure~\ref{fig:time} details the time costs (under settings from Table~\ref{tab:fidelity_evaluation}),  \slice is by far the most efficient GNN explainer. It is  up to 6 to 55 times faster than other competitors. This efficiency is critical, as \same and \subx incur high overhead from costly Shapley value computations and tree search algorithms, respectively, and fail to complete on YC and \am (after 24 hours). Ultimately, only \slice, \pgexp, and \gmask were efficient enough to run successfully on the large AM dataset.

\eetitle{Impact of layers}. As shown in Figure~\ref{fig:layer_time}, \slice runs faster when lower source layers, confirming its reliance on data locality and $l$-hop neighbor exploration. Its progressive generation facilitates early stopping for pertinent explanations, enhancing efficiency. In contrast, \gmask remains insensitive to the choice of $l$ as it simultaneously learns masks for all layers.


\begin{figure}[tb!]
    \vspace{1ex}
    \centering
\centerline{\includegraphics[scale=0.2]{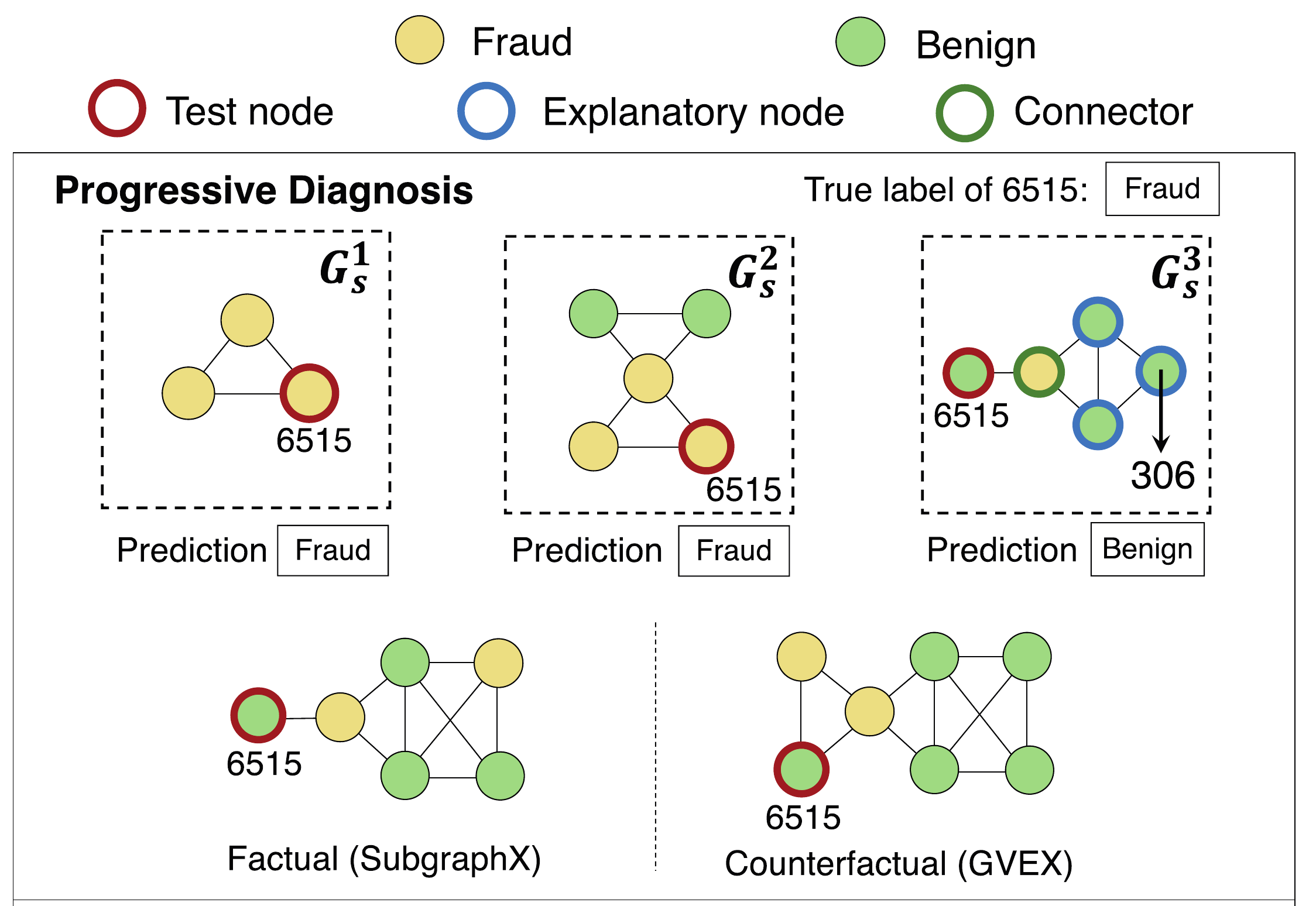}}
\vspace{-2ex}
    \caption{Case study on progressive diagnose mode}
        \vspace{-1ex}
    \label{case2}
\end{figure}

\stitle{Exp-3: Case study.} 
We next showcase the application of \slice for diagnose mode. As shown in Figure~\ref{case2}, assume a user aims to find ``why'' a \gcn $\M$ misclassified node 6515 (with true label ``Fraud'') as prediction ``Benign'' and ``when'' (at which layer) it fails (Figure~\ref{case2}). Using \slice in ``progressive diagnosis'' mode, explanations $G^1_s$, $G^2_s$, and $G^3_s$ reveal that $\M$ remains correct up to the second layer. Comparing $G^2_s$ and $G^3_s$, the error likely stems from the node 306 (``Benign'') at hop 3, leading to a failure at layer 3. For model debugging, this helps fine-tune layer 3  or switching to a 2-layer \gnn to improve accuracy. \slice can provide finer-grained, node- and layer-level explanations, whereas \subx~\cite{yuan2021explainability} and GVEX~\cite{chen2024view} -- two state-of-the-art \gnn explainability methods -- only generate a singular, final output explanation. This limitation hinders error tracing through layers, complicating misclassification diagnosis.

\section{Conclusion}
\label{sec:concl}

We introduced \slice, a layer-wise, novel \gnn explanation framework that leverages model slicing to generate explanations from selected source layers for output at target layers. We formalized explanation generation as a novel bi-criteria optimization problem, and introduced novel and efficient algorithms for both single-node, single-source and multi-node, multi-source settings. 
Our theoretical and experimental results demonstrate the effectiveness, efficiency, and scalability of \slice algorithms, as well as their practical utility for model debugging in real-world applications.

\begin{acks}
This work was supported by NSFC Grant No. 62502434 and the Ningbo Yongjiang Talent Introduction Programme (2022A-237-G). 
\end{acks}
\clearpage
\balance

\appendix
\section{Discussion}
\label{sec-discuss}

In this work, we propose a sliced model, which is derived from a pre-trained \gnn model. As described in Section~\ref{sec:extended-explanations}, an $l$-sliced model $\M^l$ consists of two primary components: the feature extractor (denoted as $f_1^l$) and the predictor (denoted as $f_2$). In practice, for an $l$-sliced model of a pre-trained \gnn $\M$ with $L$ layers, the $f_1^l$ component is constructed by taking the first $l$ pre-trained \gnn layers to capture hidden features from the graph data. Following the $f_1^l$ component, it connects the original MLP (Multi-Layer Perceptron) layers of the \gnn, constituting the $f_2$ part of our $l$-sliced model. 

There are three options to construct an $l$-sliced model:

\textbf{Option 1:} Retrain both $f_1^l$ and $f_2$;

\textbf{Option 2:} Freeze pre-trained $f_1^l$ and retrain $f_2$; and

\textbf{Option 3:} Directly use pre-trained $f^l$ and $f_2$ (adopted by \slice).

Compared to the first two options, we favor option 3 in this paper for three reasons: (a) The original model $\M$ can be treated as a black box as we focus solely on integrating layers without knowledge of internal workings of $\M$; (b) it saves time when directly utilizing the pre-trained components instead of engaging in time-consuming retraining processes; and (c) even if choosing option 1 or option 2 to retrain the model, inconsistencies in the final outputs could still exist. 

As option 3 also generates a ``new'' model $\M^l$ with potential inconsistencies, we introduce {\em guard condition} for explanation at layer $l$, including "factual" and "counterfactual," to mitigate this issue (Section~\ref{sec:extended-explanations}). For instance, by ensuring that the prediction of a test node $v_t$ using an explanation subgraph $G^l_s$ in the sliced model $\M^l$ is consistent with the prediction of $v_t$ using original graph $G$ in both $\M$ and $\M^l$, the explanation generated by option 3 remains ``faithful'' to the original model, even without any retraining. This is the rationale behind the necessity of the guard condition. However, as both Option 1 and Option 2 are more likely to violate the guard condition, it becomes difficult to generate high-quality explanations that stays faithful to the original model’s behavior. The main reasons are as follows: Although option 1 enables the sliced model $\M^l$ to fit the original data better, it poses a substantial risk of altering the original model's behavior. As our primary objective is to generate explanations based on the original \gnn model, retraining both components would create a ``new'' model that may not capture the characteristics of the original architecture and potentially change the inference process. Similarly, option 2 might make the new sliced model $\M^l$ achieve closer accuracy to the final output of the original \gnn, but it still introduces inconsistencies in model inference.

Based on derived sliced models, the design of \slice is inherently grounded in both the predictions of the original model and the sliced model in terms of guard condition, rather than treating the sliced model as an independent entity. However, for existing methods, such as \gnnexp~\cite{ying2019gnnexplainer} and \subx~\cite{yuan2021explainability}, even when applying directly to the sliced model, their explanations are not high-quality to provide truly progressive explanations that are faithful to both original model and sliced model, as they neglect the critical logical connection between the sliced model and the original model and are only accountable to the sliced model. 

\section{Proofs}
\label{sec-appendix}
\stitle{B.1 Proof of Theorem~\ref{hardness}}. 
{\em The decision problem of Explanation generation with model slicing is \NP-hard.}

\begin{proof}
We perform a reduction from the maximum coverage problem~\cite{khuller1999budgeted}. Given a collection of sets $\S$ = $S_1, \ldots S_n$, and an integer $k$, the maximum coverage problem is to find a subset $\S'\subseteq \S$ such that $|\S'|\leq k$, and the number of elements 
 $|\bigcup_{S\in\S'}S|$ is maximized (or $\geq b$ for a given threshold $b$, for its decision version).  
 
 Given $k$, $\S'$, and $\S$, 
 we construct a configuration as follows. 
 (1) We construct a graph $G$ with $V$ = $\{v_t\}\cup V_1 \cup V_2$, where 
$v_t$ is a designated target node;  for each set $S_i\in \S$, there is a node $v_{S_i}$ in $V_1$; and 
for each element $x_j\in S_i\in \S$, 
 there is a distinct node $v_{j}$ in $V_2$. 
 (2) For every element $x_{j}\in S_i$, 
 there is 
 an edge $(v_{j}, v_{S_i})$ in $E$. 
 For every node $v_{S_i}$, there is an edge between $v_{S_i}$ and $v_t$. 
 (3) Assign all the nodes in $G$ a same ground truth label. Train a $2$-layer vanilla \gnn $\M$ that assigns a correct label to $v_t$ (which is in \PTIME~\cite{chen2020scalable}). As $\M$ is a fixed model that takes as the same input $G$ as both training 
and test data, the inference ensures the invariance property of \gnns and  assigns $\M(G,v)$ with ground truth label for every 
 $v\in V$. For each 
 node $v_{S_i}$, the relative influence set $c(v_{S_i})$ is exactly the set of 
 nodes $\{v_{ij}\}$, which corresponds to the set $S_i$ and its elements. 
 Then there exists a size $k$ factual explanation 
 of $v_t$ with explainability at least $b$,  if and only if there is a solution for 
 maximum coverage with $|\bigcup_{S\in\S'}S|\geq b$. 
 As the maximum coverage problem is known to be NP-hard~\cite{khuller1999budgeted}, 
our problem is \NP-hard. 
\end{proof}

\stitle{B.2 Proof of Lemma~\ref{lm-submodular}}. 
{\em Given an output $\M(G, v)$ to be explained, 
and an explanation $G^l_s$ at layer $l$, the explainability measure $f(G^l_s)$ is 
a monotone submodular function for the node set of $G^l_s$.}

\begin{proof}
As $f(G_s^l)$ is the sum of two node set functions $I(V^l_s)$ and $D(V^l_s)$ defined in Equation~\ref{quality}, we prove that the two functions are both monotone submodular. We first analyze the property of $I(V^l_s)$. \textbf{(1)} It is clear that the relative influence set $c(v)$ for any $v\in V^l_s$ is non-negative. So $I(V^l_s)$ is non-decreasing monotone function. \textbf{(2)} Next we will show that for any set $V_{s'} \subset V_{s}$ and any node $u \notin V_{s}$, $I(\cdot)$ is submodular by verifying the following function: 
\vspace{-1mm}
\begin{equation}
    \nonumber
    I(V_{s'} \cup \{u\} ) - I(V_{s'}) \ge I(V_{s} \cup \{u\} ) - I(V_{s}) 
\end{equation}

If $c(\{u\}) \cap c(V_{s}) = \emptyset$, we have $I(V_{s} \cup \{u\} ) = I(V_{s}) + I(u)$ and $I(V_{s'} \cup \{u\} ) = I(V_{s'}) + I(u)$, which satisfies the equation. \textbf{Case (ii)}. Otherwise, we assume that $I(V_{s'} \cup \{u\} )$ = $I(V_{s'}) + \triangle  I(u|V_{s'})$ and $I(V_{s} \cup \{u\} )$ = $I(V_{s}) + \triangle I(u|V_{s})$, where $\triangle  I()$ indicates the marginal gain by adding the new element $u$. Due to the non-negative nature of $c(V_{s}\backslash V_{s'})$, we have $\lvert c(u) \cap c(V_{s}\backslash V_{s'}) \rvert \ge 0$, which indicates that $\triangle  I(u|V_{s'}) \ge \triangle  I(u|V_{s})$, thus satisfying the inequality. 

Following a similar analysis, we can show that $D(\cdot)$ is also monotone submodular. To prove that, we need to show that the function satisfies the diminishing returns property. Considering the same condition $V_{s'} \subset V_{s}$ and any node $u \notin V_{s}$, $D(V_{s'} \cup \{u\}) - D(V_{s'}) \ge D(V_{s} \cup \{u\}) - D(V_{s})$. According to union property, since $V_{s'} \subset V_{s}$, the contribution of $r(u) $ to the union $\bigcup_{v \in V_{s'} \cup \{u\}} r(v) $ will be greater than or equal to its contribution to $\bigcup_{v \in V_{s} \cup \{u\}} r(v)$ due to overlapping elements in $ r(u) $ that are already covered by $ V_{s} $. This means that the increase in $ \left|\bigcup_{v \in V_{s'} \cup \{u\}} r(v)\right| $ from adding $ u $ to $ V_{s'} $ is at least as large as the increase in $ \left|\bigcup_{v \in V_{s} \cup \{u\}} r(v)\right| $ from adding $ u $ to $ V_{s} $, satisfying the diminishing returns property required for submodularity. Putting these together, the bi-criteria explainability measure $f(G_s^l)$ is a monotone submodular function. 

\end{proof}

\stitle{B.3 Proof of Theorem~\ref{theor-approx}}.
{\em
Given a configuration $\C$, with a specific test node $v$ and a specific \gnn layer $l$, there is a $\frac{1}{2}$-approximate algorithm for generating model-slicing explanation.
}

\begin{proof}
 Given a configuration $\C$, with a specific test node $v$ and a specific \gnn layer $l$, we reformulate the problem as node selection problem, analogous to a counterpart as a Max-Sum Diversification task, with a modified facility dispersion objective $f'(V_s) = \gamma \sum_{v \in V_s^l} |c(v)| + (1 - \gamma) \sum_{\substack{u, v \in V_s^l }} \mathbf{1}[d(Z_u^l,Z_v^l) \ge \theta]$. Thus, given the approximation preserving reduction, we can map known results about Max-Sum Diversification to the diversification objective. We observe that maximizing the objective is NP-Hard, but  there are known approximation algorithms for the problem~\cite{hassin1997approximation}, which yields a \( \frac{1}{2} \)-approximation. 
\end{proof}

\vspace{-2mm}
\section{SliceMS and SliceMM}
\label{sec-ms}
\stitle{C.1 Multi-nodes, Single-Source Layers}. 
{\em The decision problem of Explanation generation with model slicing is \NP-hard even for a single source layer ($|\L|$ = $1$).}
In \slicess, the explanations are generated for each test node sequentially, which is inefficient. To address this limitation, \slicems employs an incremental approach, simultaneously selecting explanatory nodes for all test nodes $V_T$ during {\em Generation Phase} in Algorithm \ref{alg:greedy}.


\stitle{Algorithm}. The algorithm, denoted as
\slicems and illustrated as Algorithm \ref{alg:greedy2}. 

\eetitle{Initialization} (lines 1-6). Algorithm~\ref{alg:greedy2} starts with setting $G^l(V_T)$ as all potential explanatory nodes for $V_T$ because of data locality of \gnn inference. Next, it introduces entry $\V_s[i]$ to store the ``current'' explanatory node set of the $i$-th element in $V_T$, and a boolean value $v.B[i]$ to represent whether a potential explanatory node $v \in G^l(V_T)$ is likely to join $\V_s[i]$. For $i$-th element in $V_T$, denoted as $v_{t_i}$, $\V_s[i]$ is initialized to an empty set, indicating that no explanatory nodes have been selected for $v_{t_i}$, and each node $v \in G^l(v_{t_i})$ is a ``candidate'' explanatory node for $v_{t_i}$, i.e., $v.B[i]=true$. In addition, it maintains a global candidate explanatory node set $V_s'$ (line 6).

\eetitle{Incremental generation phase} (lines 7-17). ~\slicems follows a pairwise node selection based on a greedy strategy similar to \slicess. When a new node $v\in \{v_1^*,v_2^*\}$ is introduced for maximum marginal gain corresponding to at least one test node, it uses $v$ to update explanatory node set{s $\V_s[]$. Specifically, if $v.B[i]$ is true, it indicates that $v$ is a ``promising'' explanatory node for the $i$-th test node. Thus, the explanatory sets of such test nodes influenced by $v$ can be updated all at once. When an explanatory node set of a test node in $\V_s[]$ is updated, it generates a subgraph $G_s$ with connectivity. Next, it invokes Procedure \verifyms to test if $G_s$ is an explanation for multiple test nodes in $V_T$.

\stitle{Procedure} \verifyms. The procedure verifies if the subgraph $G_s$ induced by explanatory node set $V_s$ serves as an explanation for multiple test nodes in $V_T$. Specifically, for all node $v\in V_s$, if $v.B[i]$ is true, which means that $V_s$ can serve as a candidate explanation for the $i$-th test node $v_{t_i}$, then it invokes \verifyss to verify if $G_s$ is an explanation for test nodes that satisfy the verification condition in Section~\ref{sec:extended-explanations}. If so, it updates $\V_s[i]$ as the nodes in $G_s$ and reduces $V_T$ to retain only those that have no explanation to further explore explanations for them in next loop. 


\stitle{C.2 Multi-nodes, Multi-Source Layers}. 
We describe the \slicemm procedure, illustrated in Algorithm~\ref{alg:greedy3}. \slicemm processes the layers $\mathcal{L}$ sequentially, starting from the last layer and moving downward (i.e., layer 1). For each layer, it performs two main steps:  (1) It executes a ``hop jumping'' process that prunes the explanatory node set by removing nodes beyond $l$ hops when at layer $l$. This layer-by-layer updating and pruning enhance efficiency by narrowing the search space to relevant nodes, thus avoiding unnecessary computations on irrelevant nodes for the current layer's output. (2) It updates a map $\mathcal{V}_B$ using \slicems, where each entry $\mathcal{V}_B[i][j]$ contains the best explanatory node set $V^{l_j}_s$ for the output $\mathcal{M}^{l_t}(G,v_{t_i})$ concerning $v_{t_i} \in V_T$ at layer $l_j$.
\floatname{algorithm}{Algorithm}

\begin{algorithm}[tb!]
    \renewcommand{\algorithmicrequire}
    {\textbf{Input:}}
    \renewcommand{\algorithmicensure}
    {\textbf{Output:}}
    \caption{:~\slicems (multi nodes, single source layer)
    }
    \begin{algorithmic}[1]
        \REQUIRE a configuration $\C$ = $(G,\M, V_T, \{l\}, l_t, k)$;  
        \ENSURE a set of explanations $\G^l_s$ for each $v_t\in V_T$ \wrt 
        $\M^{l_t}(G,v_t)$ at \gnn layer $l$. 
        \STATE $G^l(V_T)=\bigcup_{v_t\in V_T}{G^l(v_t)}$ // union $l$-hop of $V_T$
        \STATE Set $\V_s[i]=\emptyset$ for $i\in [0,|V_T|-1]$
        \FOR{$i\in [0,|V_T|-1]$}
        \FOR{$v\in G^l(v_t)$, where $v_t$ is the $i$-th element in $V_T$,}
            \STATE set $v.B[i]=true$ // a boolean indicator
        \ENDFOR
        \ENDFOR
        \STATE set $V_s'=$ $G^l(V_T)$
        \FOR{$i\in [0,|V_T|-1]$}
        \WHILE{$|\V_s[i]| < k$ and $V_s'\neq\emptyset$}
            \STATE select $\{v^*_1$, $v^*_2\}\in V_s'$ using pair-wise node selection in \slicess(line 4) based on $i$-th element in $V_T$
            \FOR{$j \in [0,|V_T|-1]$}
                \IF{$v_1^*.B[j] = \text{true}$}
                    \STATE update $\V_s[j]=\V_s[j]\cup \{v_1^*\}$
                \ENDIF
                \IF{$v_2^*.B[j] = \text{true}$}
                    \STATE update $\V_s[j]=\V_s[j]\cup \{v_2^*\}$
                \ENDIF
            \ENDFOR
            \STATE \verifyms($\V_s[i],C$)
        \STATE $V_s' = V_s'\backslash \{v_1^*,v_2^*\}$

        \ENDWHILE
        \ENDFOR
        \RETURN $\G^l_s$. 
    \end{algorithmic}
    \label{alg:greedy2}
\end{algorithm}
\vspace{-3mm}

\begin{algorithm}[tb!]
    \renewcommand{\algorithmicrequire}
    {\textbf{Input:}}
    \renewcommand{\algorithmicensure}
    {\textbf{Output:}}
    \caption{:~\slicemm (multi nodes, multi source layers)}
    \begin{algorithmic}[1]
        \REQUIRE a configuration $\C$ = $(G,\M, V_T, \L, l_t, k)$;  
        \ENSURE a set of explanations $\G_s=\{\G^1_s,...,\G^L_s\}$ for each $v_t\in V_T$ and layer $l\in \L$ \wrt 
        $\M^{l_t}(G,v)$ at \gnn layer $l$. 
        \STATE $G^L(V_T)=\bigcup_{v_t\in V_T}{G^L(v_t)}$
        \STATE $\forall v_{t_i}\in V_T$, $\forall l_j\in \L$, set $\V_B[i][j]=\emptyset$ 
        \STATE sort $\L$ in a decreasing order
        \FOR{$l_j \in \L$}
            \STATE invoke \slicems to update $V_B[][j]$ \wrt layer $l_j$
            \STATE set $\G^{l_j}_s$ as a set of subgraphs derived from $V_B[][j]$ 
            \FOR{$v_{t_i}\in V_T$}
                \IF{$v_{t_j}$ has no explanation at layer $l_j$ \textbf{and} $l_j \neq \L$}
                    \STATE $V_B[i][j+1]=V_B[i]\backslash \{v|v\notin N^{l_{j+1}}(v_{t_i})\}$
                \ENDIF
            \ENDFOR
        \ENDFOR
        \RETURN $\G_s$. 
    \end{algorithmic}
    \label{alg:greedy3}
\end{algorithm}

 \begin{figure*}[tb!]
    \vspace{-4mm}
    \centering
    \begin{minipage}[]{0.4\linewidth}
        \centering
        \includegraphics[width=\linewidth]{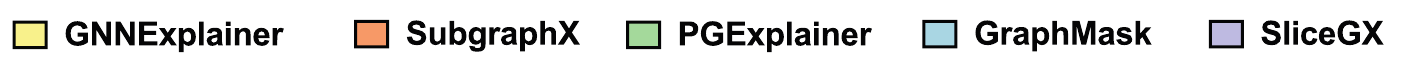}
    \end{minipage}
    \hspace{5mm}  
    \begin{minipage}[]{0.55\linewidth}
        \centering
        \includegraphics[width=\linewidth]{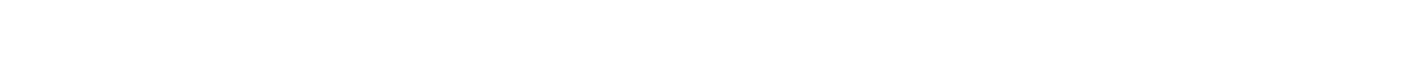}
    \end{minipage}
\end{figure*}
\begin{figure*}[tb!]
	\vspace{-8mm}
	\centering

    \subfigure[Fidelity+ over \gnns]
	{\includegraphics[width=0.2\linewidth]{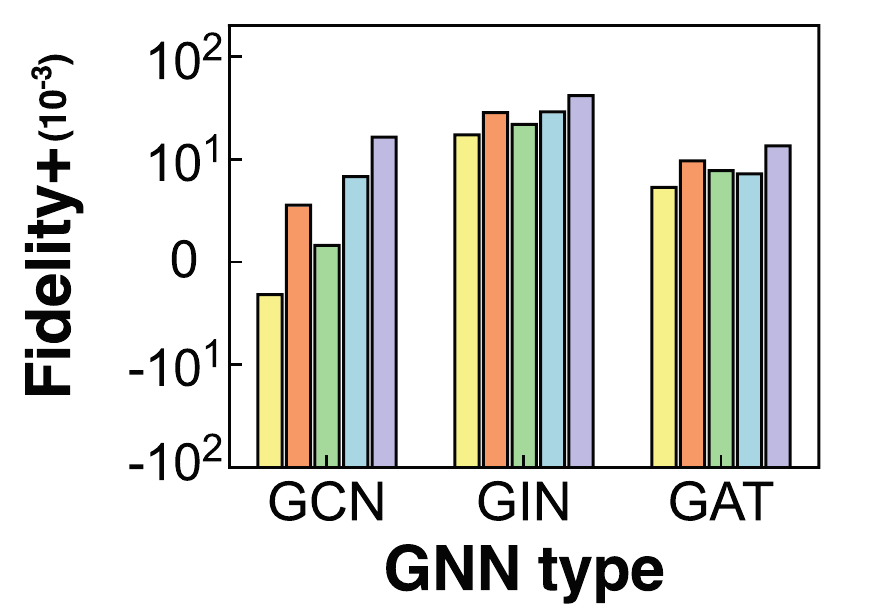}
		\label{fig:gnntype-fplus}}
    \hspace{-3.7mm}
    \subfigure[Fidelity- over \gnns]
	{\includegraphics[width=0.2\linewidth]{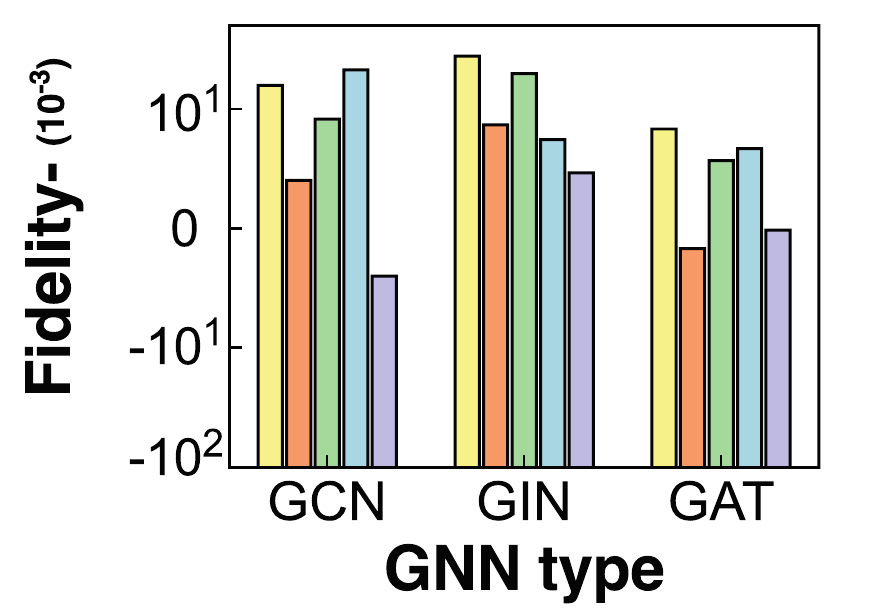}
	    \label{fig:gnntype-fminus}}
    \hspace{-3.7mm}    
    \subfigure[Fidelity+ over layer]
	{\includegraphics[width=0.2\linewidth]{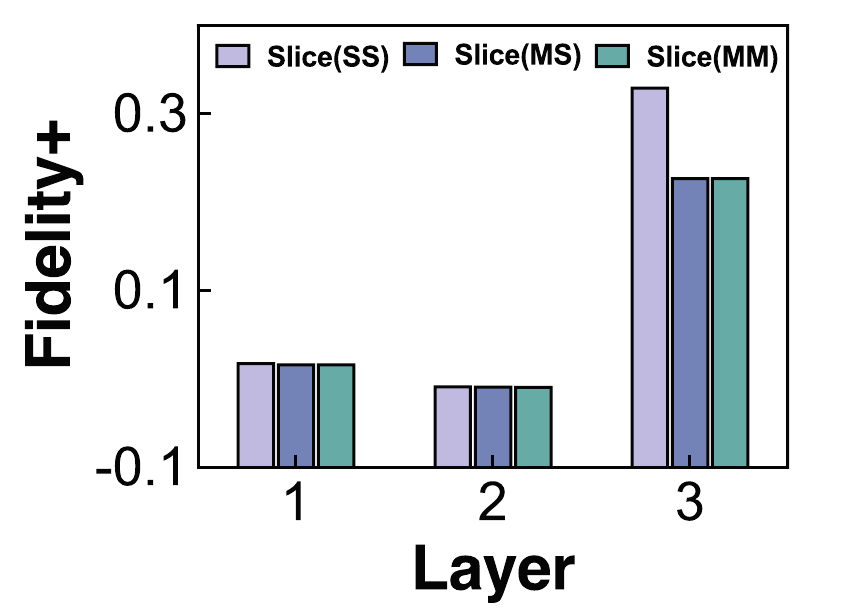}
		\label{fig:slice-fplus}}	
    \hspace{-3.7mm}
    \subfigure[Fidelity- over layer]
	{\includegraphics[width=0.2\linewidth]{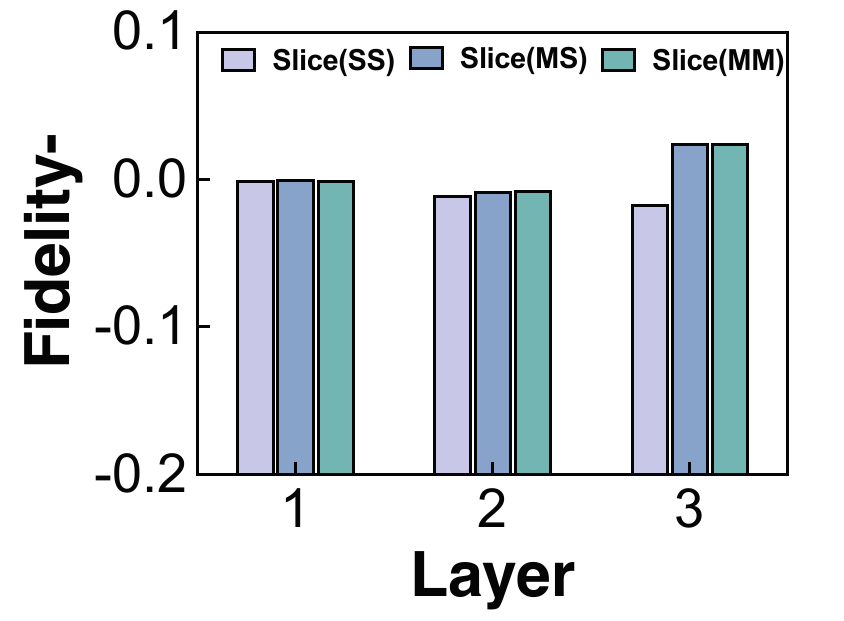}
		\label{fig:slice-fminus}}
    \hspace{-3.7mm}
    \subfigure[Runtime($|V_T|=100$)]
	{\includegraphics[width=0.2\linewidth]{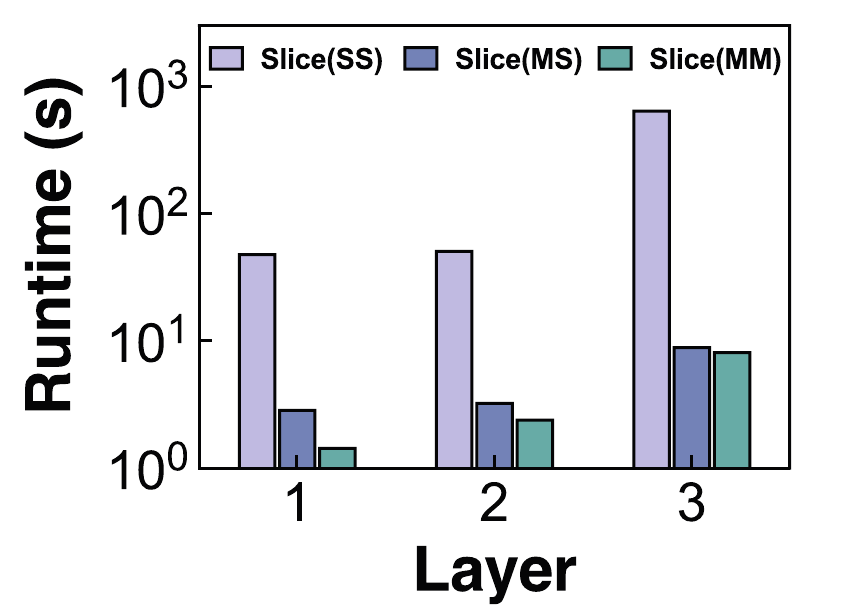}
		\label{fig:slice-time}}
	
	\par
    \hspace{-8mm}
    \textbf{\small Group1: Fidelity (different types of \gnns)}
    \hspace{25mm}
    \textbf{\small Group2: Performance of \slicems and \slicemm)}\\
    \vspace{-1mm}
    \caption{Additional experimental results}
	\label{fig:slice}
	\vspace{-2.5mm}
\end{figure*}

\section{Datasets}
\label{sec-ds}
We adopt six benchmark graph datasets~\cite{fey2019fast}, as summarized in 
Table~\ref{tab-data}. 
(1) \textbf{BA-shapes (\ba)}~\cite{fey2019fast} is a synthetic 
Barabási-Albert (BA) graphs~\cite{fey2019fast}. The graph 
is generated by augmenting a set of 
nodes with structured subgraphs (motifs). 
A benchmark task for BA is node classification, 
where the nodes are associated with 
four numeric classes. Since the dataset does not originally provide node features, we follow PGExplainer~\cite{luo2020parameterized} to process the dataset and derive the node features as 10-dimensional all-one vectors. 
(2) \textbf{Tree-Cycles (TR)~\cite{fey2019fast}} 
are synthetic graphs obtained by 
augmenting 8-level balanced binary trees with six-node cycle motifs. The dataset includes two classes: one for the nodes in the base graph and another for the nodes within the defined motifs. Similar to BA-shapes, we process the dataset to derive the node features.
(3) \textbf{Cora (CO)~\cite{kipf2016semi} } is a widely-used citation network. The nodes represent research papers, while the edges correspond to citation relationships. Each node is labeled with a topic for multi-label classification analysis.
\textbf{(4) Coauthor CS~\cite{shchur2018pitfalls} (CS)} is a co-authorship graph based on the Microsoft Academic Graph. The nodes include authors and papers, interlinked with co-authorship relations. The node features represent paper keywords, and each author is assigned a class label  indicating his or her most active field of study. 
\textbf{(5) YelpChi(YC)~\cite{dou2020enhancing} (CS)} is a binary classification dataset designed for spam reviews detection. It consists of product reviews categorized as “Benign User”
or “Fraudster”. Each node represents a review, while edges capture relationships between reviews, such as shared users, products, re-
view text, or temporal connections. 
\textbf{(6) Amazon~\cite{zeng2019graphsaint} (AM)} is a co-purchasing network, 
where each node is a product, and an edge between two products means they are co-purchased.
The text reviews of the products are processed to generate 4-gram features, which are reduced to 200-dimensional vectors using SVD, serving as the node features. The labels represent product categories, \eg books, movies.

\begin{table}[tb!]
    \centering
     \caption{Statistics and properties of datasets }
     \vspace{-1mm}
    \resizebox{\linewidth}{!}{
    \begin{tabular}{@{}c|ccccc@{}}
         \toprule
         {\textbf{Category}} & {\textbf{Dataset}} & {\textbf{\# nodes}} & {\textbf{\# edges}} & {\textbf{\# node features}} & {\textbf{\# classes}} \\
         \midrule \midrule
         \multirow{2}{*}{\textbf{Synthetic}} & {\textbf{BA-shapes}} & 700 & 4110 &10 & 4\\
         & {\textbf{Tree-Cycles}} & 871 & 1950 &10 & 2\\
         \midrule
         \multirow{4}{*}{\textbf{Real-world}} & {\textbf{Cora}} & 2708 & 10556 &1433 & 7\\
                    & {\textbf{Coauthor CS}} & 18333 & 163788 &6805 & 15\\
                    & {\textbf{YelpChi}} & 45954 & 7693958 & 32 & 2\\
                    & {\textbf{Amazon}} & 1569960 & 264339468 &200 & 107\\
         \bottomrule
    \end{tabular}
    }
    \label{tab-data}
\end{table}

\section{Additional Experimental Results}
\label{sec-exp}

\stitle{E.1 Different types of \gnns}. We showcase that \slice can be also applied to different types of \gnns, including GCN~\cite{kipf2016semi}, GIN~\cite{xu2018powerful}, and GAT~\cite{velivckovic2017graph}. In experimental setting, the hyperparameters of GIN, including the learning rate and training epochs, closely follow the setting of GCN. For GAT, we use 10 attention heads with 10 dimensions each, and thus 100 hidden dimensions. 
 
\stitle{E.2 Performance of \slicems and \slicemm}. We evaluated the overall performance and efficiency over layers by two extended algorithms of \slice (\slicess): \slicems and \slicemm. Figs.~\ref{fig:slice-fplus} and~\ref{fig:slice-fminus}
illustrate the overall performance including fidelity+ and fidelity-, and Figure~\ref{fig:slice-time} presents the time for generating explanations over layers. The results indicate that compared to \slice, \slicems achieves faster generation speed while almost maintaining original performance due to its global updating approach. Meanwhile, \slicemm utilizes a "hop jumping" process and eliminates some redundant greedy selection steps, which is the most time-efficient.

This is a desirable 
property for online explanation generation 
of \gnn applications over data streams such as traffic 
analysis, social networks, and transaction networks.

\end{document}